\documentclass{article}

    \PassOptionsToPackage{numbers, compress}{natbib}
\usepackage[preprint]{neurips_2020}
\pdfoutput=1
\usepackage[utf8]{inputenc} % allow utf-8 input
\usepackage[T1]{fontenc}    % use 8-bit T1 fonts
\usepackage[bookmarks=false]{hyperref}       % hyperlinks
\usepackage{url}            % simple URL typesetting
\usepackage{booktabs}       % professional-quality tables
\usepackage{amsfonts}       % blackboard math symbols
\usepackage{nicefrac}       % compact symbols for 1/2, etc.
\usepackage{microtype}      % microtypography
\usepackage{color}

\newcommand{\squishlist}{
   \begin{list}{$\bullet$}
    { \setlength{\itemsep}{0pt}      \setlength{\parsep}{3pt}
      \setlength{\topsep}{3pt}       \setlength{\partopsep}{0pt}
      \setlength{\leftmargin}{1.5em} \setlength{\labelwidth}{1em}
      \setlength{\labelsep}{0.5em} } }
\newcommand{\squishend}{  \end{list}  }

\usepackage{amsmath,amssymb,xspace,verbatim,paralist,enumitem}
\usepackage{array,multirow,multicol}
\usepackage{algorithm,algorithmic}
\usepackage[lofdepth,lotdepth]{subfig}
\usepackage{graphicx}
\usepackage{tikz}
\usetikzlibrary{calc}

\usepackage{thmtools}
\usepackage{thm-restate}
\usepackage[normalem]{ulem} % For strking out text
\setlist[enumerate]{nosep}
\usepackage[nameinlink]{cleveref}
\usepackage{multicol,caption}

\newbox\mybox 
\newdimen\myboxwidth    

\newcommand\addpicture[3]{% 
\setbox\mybox=\hbox{\includegraphics[scale=#3]{#2}}
\myboxwidth\wd\mybox    
\renewcommand\windowpagestuff{% 
\includegraphics[scale=#3]{#2}
\captionof{figure}{A test figure.}}
\parpic[#1]{% 
\begin{minipage}{\myboxwidth}
 \windowpagestuff 
\end{minipage} 
} }

\newcommand{\ignore}[1]{}

\newcommand{\cA}{{\cal A}}

\newcommand{\cC}{{\cal C}}

\newcommand{\cI}{{\cal I}}

\newcommand{\cX}{{\cal X}}

\newcommand{\eps}{\varepsilon}

\newcommand{\calL}{{\cal L}}
\newcommand{\calF}{{\cal F}}

\newcommand{\calC}{{\cal C}}
\newcommand{\calX}{{\cal X}}

\newcommand{\calI}{{\cal I}}
\newcommand{\calJ}{{\cal J}}

\DeclareMathOperator{\opt}{OPT}

\newcommand{\mypar}[1]{\medskip\noindent{\bfseries #1.}~}

\newcommand{\etal}{{et al.}\xspace}

\DeclareMathOperator{\EM}{EM}

\newcommand\Tstrut {\rule{0pt}{3ex}}         % = `top' strut
\newcommand\Bstrut {\rule[-1.3ex]{0pt}{0pt}}   % = `bottom' strut

\newcommand{\floor}[1]{\ensuremath{\left\lfloor#1\right\rfloor}}

\newcommand{\be}{\begin{enumerate}}
\newcommand{\ee}{\end{enumerate}}
\newcommand{\bd}{\begin{description}}
\newcommand{\ed}{\end{description}}
\newcommand{\bi}{\begin{itemize}}
\newcommand{\ei}{\end{itemize}}

\newtheorem{theorem}{Theorem}
\newtheorem{lemma}[theorem]{Lemma}

\newtheorem{claim}[theorem]{Claim}

\newtheorem{definition}[theorem]{Definition}

\newtheorem{remark}{Remark}

\newenvironment{proof}{\par \smallskip{\bf Proof:}}{\hfill\stopproof}
\def\stopproof{\square}
\def\square{\vbox{\hrule height.2pt\hbox{\vrule width.2pt height5pt \kern5pt
\vrule width.2pt} \hrule height.2pt}}

\newenvironment{prog}[1]{
\begin{minipage}{5.8 in}
{\sc\bf #1}
\begin{enumerate}}
{
\end{enumerate}
\end{minipage}
}

\renewcommand{\phi}{\varphi}

\newcommand{\rs}{R^{\star}}

\DeclareMathOperator{\hkm}{HKM}
\DeclareMathOperator{\skm}{SKM}
\DeclareMathOperator{\algif}{ALG-IF}
\DeclareMathOperator{\alggf}{GF}
\DeclareMathOperator{\algtf}{ALG-CF}
\DeclareMathOperator{\optif}{OPT-IF}
\DeclareMathOperator{\opttf}{OPT-CF}

\DeclareMathOperator{\supp}{supp}
\newcommand{\tf}{combined fairness}

\DeclareMathOperator*{\argmax}{argmax}

\DeclareMathOperator{\MAD}{max-violation}
\DeclareMathOperator{\bias}{bias}
\DeclareMathOperator{\TV}{TV}

\DeclareMathOperator{\KL}{KL}

\newcommand{\ALG}{\textsc{ALG-IF}\xspace}
\newcommand{\fairness}{{\cal F}}
\newcommand{\iopt}{\opt_{k,p,f,\fairness}}
\newcommand{\vopt}{\opt_{k,p}}
\newcommand{\total}{{\sc combined fair $(k,p,f,\fairness)$-clustering }}
\newcommand{\Real}{{\mathbb R}}
\newcommand{\Ex}{{\mathbb E}}
\newcommand{\kfp}{{\sc individually fair $(k,p,f,\fairness)$-clustering }}
\newcommand{\kcf}{{\sc individually fair $(f,\fairness)$ $k$-center }}
\newcommand{\vkp}{{\sc vanilla $(k,p)$-clustering }}
\newcommand{\vkc}{{\sc vanilla $k$-center }}
\newcommand{\ifa}{{\sc individually fair $p$-assignment }}

\newcommand{\fairopt}{\textsc{fair-assgn }}
\newcommand{\fairoptkc}{\textsc{fair-assgn-kc }}
\newcommand{\LPB}{\textsc{LP-Bias }}
\newcommand{\LPBD}{\textsc{LP-Bias-Dual }}
\newcommand{\LPEM}{\textsc{LP-EM }}
\newcommand{\gr}{|G_r|}
\newcommand{\totalassgn}{\textsc{combined-fair-assgn }}

\title{Distributional Individual Fairness in Clustering}

\author{%
  Nihesh Anderson \\
  Department of Computer Science\\
  IIIT Delhi\\
   \And
   Suman K. Bera \\
  Department of Computer Science\\
  UC Santa Cruz\\
   \AND
   Syamantak Das \\
   Department of Computer Science\\
   IIIT Delhi \\
   \And
   Yang Liu \\
   Department of Computer Science\\
   UC Santa Cruz \\
}

\begin{document}

\maketitle

\begin{abstract}
 In this paper, we initiate the study of fair clustering that ensures distributional similarity among similar individuals. In response to improving fairness in machine learning, recent papers have investigated fairness in clustering algorithms and have focused on the paradigm of statistical parity/group fairness. These efforts attempt to minimize bias against some protected groups in the population. However, to the best of our knowledge, the alternative viewpoint of individual fairness, introduced by Dwork et al. (ITCS 2012) in the context of classification, has not been considered for clustering so far. Similar to Dwork et al., we adopt the individual fairness notion which mandates that similar individuals should be treated similarly for clustering problems. We use the notion of $f$-divergence as a measure of statistical similarity that significantly generalizes the ones used by Dwork et al. We introduce a framework for assigning individuals, embedded in a metric space, to probability distributions over a bounded number of cluster centers. The objective is to ensure (a) low cost of clustering in expectation and (b) individuals that are close to each other in a given fairness space are mapped to statistically similar distributions.

We provide an algorithm for clustering with $p$-norm objective ($k$-center, $k$-means are special cases) and individual fairness constraints with provable approximation guarantee. We extend this framework to include both group fairness and individual fairness inside the protected groups. Finally, we observe conditions under which individual fairness implies group fairness. We present extensive experimental evidence that justifies the effectiveness of our approach.

\end{abstract}

\section{Introduction}

% 1. Fairness in ML in general
Increasing deployment of machine learning based systems in decision making tasks such as targeted ad placement~\cite{speicher2018potential}, issuing home loans~\cite{bartlett2019consumer},
predicting recidivism~\cite{propublica,chouldechova2017fair}, and gender inequality at workplace~\cite{datta2015automated,Miller2015}  mandates that such algorithms are fair to individuals or groups in a population. An increasing body of research over the last decade has attempted to define various notions of fairness in such systems and design efficient learning algorithms that respect these fairness constraints (see the excellent survey by Mehrabi~\etal~\cite{mehrabi2019survey}). 

Clustering is a classical unsupervised learning technique with wide applications in domains such as recommender systems \cite{sarwar2002recommender}, customer segmentation \cite{chen2012data}, feature generation \cite{larsen1999fast,kg2006feature}, %\yl{the 2006 paper particularly mentioned fuzzy k means for feature extraction}
targeted advertisement \cite{aggarwal2004method}, etc. The seminal work of  Chierichetti et al.~\cite{CKLV18} initiated the study of {\em group} fairness (also called statistical fairness) in clustering. Group fairness requires that the representation of various protected groups in all the clusters should be balanced. The work of~\cite{CKLV18} was immediately followed up by several researchers~{\cite{RS18,Bercea2018,Backurs2019,bera19,Ahmadian:2019,huang2019coresets} leading to efficient algorithms for a wide variety of clustering problems under group fairness constraints. 
	
In this paper, we consider the alternate viewpoint of {\em individual fairness} introduced in the influential work of Dwork et al.~\cite{dwork2012fairness} in the context of classification problems. To the best of our knowledge, this particular notion of individual fairness has not been previously studied for clustering problems. Our main motivation is to address the possibility of standard clustering algorithms or clustering algorithms enforcing group fairness being unfair to `similar' individuals, as illustrated by~\Cref{fig:unfairBoundary} and~\Cref{fig:combined_fairness}. Taking~\Cref{fig:combined_fairness} for example, 
group fairness demands that, in each cluster, roughly one-third of the points must be circles (red). Let $R_L$ and $B_L$ be the sets of red and blue points on the left respectively. Naturally, two of the points from 
the set $R_{L}$ , marked with oval, needs to be assigned to the cluster $C_R$ on the right. However, this would violate {\em individual}
fairness between the points inside the oval and the remaining points in $R_{L}\cup B_L$. In fact, it has been shown that forcing {\em group} fairness can lead to {\em disparate treatment} 
of similar individuals or open up the possibility of gerrymandering by unfairly targeting a subgroup of a protected group --- see~\cite{kearns2018preventing,kearns2019empirical}.  %\yl{Add discussions back about the figures.} 

 %Among others, the two widely popular categories have emerged: {\em group} fairness (or {\em statistical} fairness)
%and {\em individual} fairness. In the former, the {\em odds} are required to be approximately equalized 
%across various protected groups in the 
%population. This solution concept has been successfully implemented to improve the fairness guarantee of a clustering algorithm      

%Machine learning based systems are increasingly being deployed in tasks that involve making decisions about individuals. In recent years, it has been demonstrated that such systems might suffer from inherent biases in favor or against a certain group of the population under consideration. 
%Prominent examples include face recognition tools~\cite{Harwell2019}, targeted ad placement~\cite{speicher2018potential}, issuing home loans~\cite{bartlett2019consumer},
%predicting recidivism~\cite{propublica,chouldechova2017fair} and gender inequality at workplace~\cite{datta2015automated,Miller2015}. 

%Clustering technique is not an exception. 

% 2. Notion of Fairness
%To address these issues, two broad categories of fairness notions have emerged: {\em group} fairness (or {\em statistical} fairness)
%and {\em individual} fairness. In the former, the {\em odds} are required to be approximately equalized 
%across various protected groups in the 
%population. This solution concept has been successfully implemented to improve the fairness guarantee of a clustering algorithm \cite{CKLV18,RS18,Bercea2018,bera19,Ahmadian:2019}. 

\begin{figure}[!ht]
	\centering
	\begin{minipage}[b]{.35\textwidth}
		\centerline{\includegraphics[height=0.8\linewidth,width=1\linewidth]{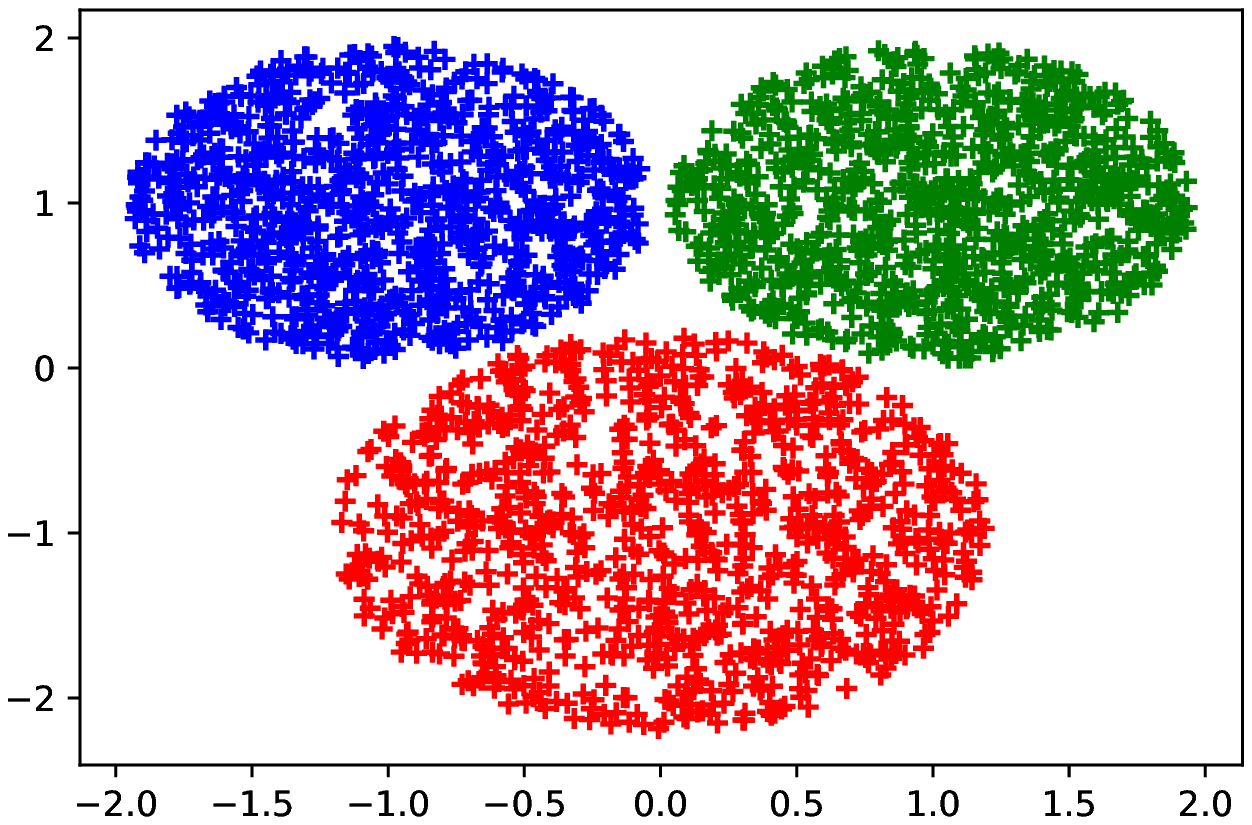}}
		\caption{A simple problem to illustrate the possibility of harming the individuals (on the boundaries of three clusters) via off-the-shelf clustering method.}
		\label{fig:unfairBoundary}
	\end{minipage}\qquad
	\begin{minipage}[b]{.5\textwidth}
		\centering
		\includegraphics[width=1\linewidth,bb={0 0 10in 10in},trim={2.5in 9.5in 7in 1in},clip=true]{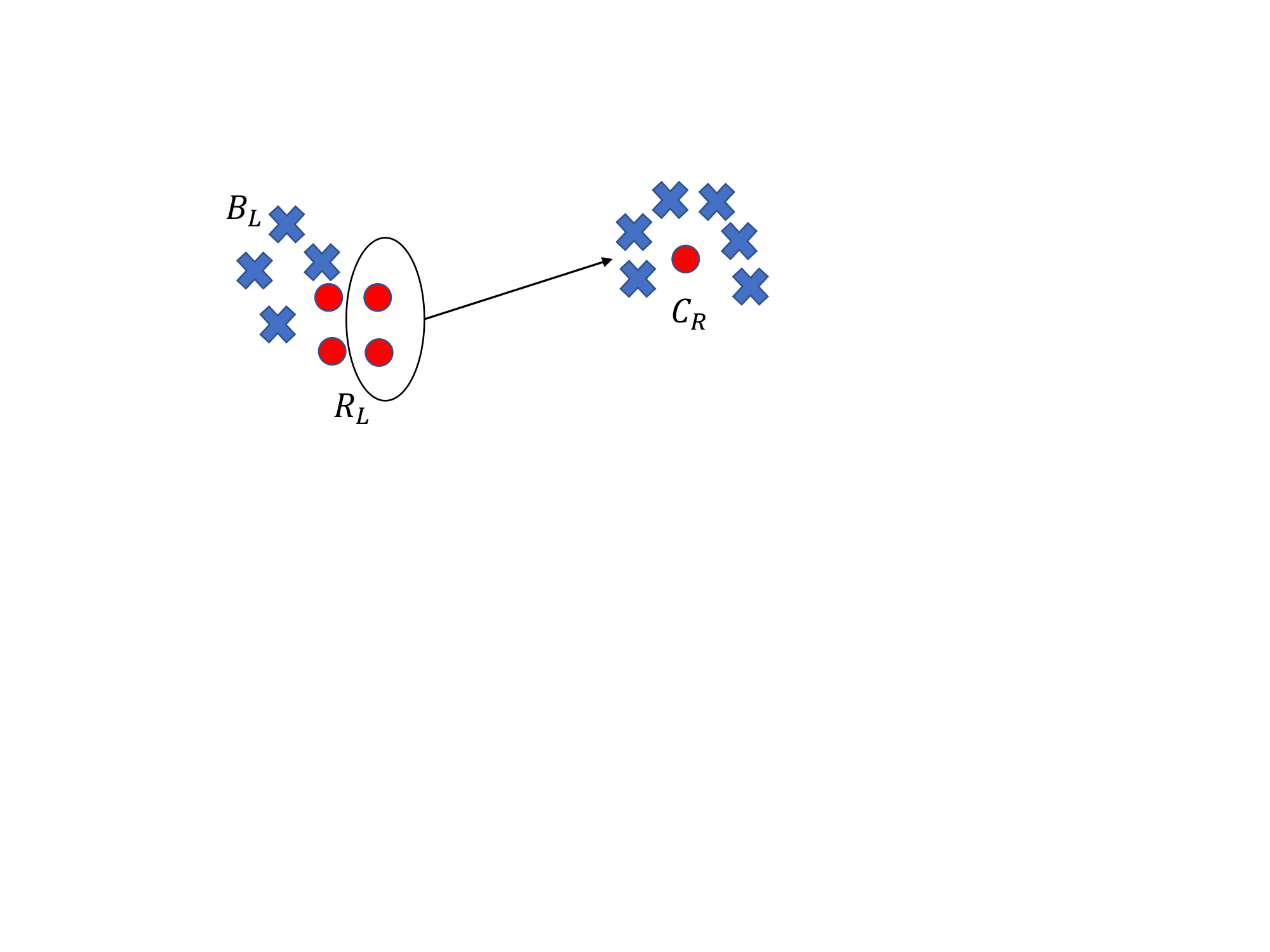}
		\caption{Group fairness might affect Individual Fairness. Moving the red points within the oval from $R_L$ to the right cluster $C_R$ would violate {\em individual}
			fairness constraints between these points and the remaining points in $R_{L}$ }
		\label{fig:combined_fairness}
	\end{minipage}
\end{figure}

\noindent
{\bf Our notion: Individual fairness in $k$-clustering.}
In $k$-clustering problems ($k$-means, $k$-median, $k$-center, etc.), the input consists of a set of
points $V$ embedded in a known metric space. The goal
is to partition the points into $k$ clusters while minimizing some distance-based objective function. 
%It is natural to expect that individuals which are `close' to each other in this space would be assigned to the same cluster by any reasonable clustering heuristic. However, as demonstrated in Fig.~\ref{fig:unfairBoundary}, any deterministic clustering method can potentially affect similar individuals that lie at the boundary of the clusters. In order to overcome this issue,
We propose a randomized assignment of points to centers as part of our solution concept. Inspired by ideas from Dwork et al.~\cite{dwork2012fairness}, our algorithm produces a set of $k$ centers denoted by $\calC$, and a mapping of each point $x\in V$ to a distribution over the $k$ centers, while minimizing the expected clustering cost. Note that this is related to probabilistic clustering solutions such as {\em soft} $k$-means~\cite{DudaHS01} or {\em fuzzy} $k$-means}~\cite{kg2006feature,bezdek2013pattern}. However, we show in our experiments that these solutions can be unfair to individuals.

We enforce individual fairness between points through distributional similarity. We assume a fairness similarity measure $\fairness:V\times V \rightarrow \Real_{\geq 0}$ (not necessarily a metric) that maps every pair of points in the population to some non-negative real number. We require the statistical distance between the output distributions of two points in $V$, measured by $f$-divergence~\cite{csiszar1964informationstheoretische,morimoto1963markov,ali1966general}, to be upper bounded by their $\fairness$-measure. %In particular, we use , as the similarity measure between two distributions. 
This is analogous to the definition of individual fairness in classification by Dwork et al.~\cite{dwork2012fairness}, where they utilize the special cases of $f$-divergence, namely, {\em total variational distance} and {\em relative $\ell_{\infty}$ metric}. However, in classification, either one has to assume the knowledge of a similarity measure as side information, or face the non-trivial task of computing~\cite{KleinbergT02} or learning the same~\cite{zemel2013learning}. On the other hand, in clustering problems, the distance metric $d$ provided by the feature space can be considered as a natural choice of the fairness similarity measure. However, we emphasize that all our results hold for any arbitrary choice of fairness similarity measure.   

\subsection{Our Contribution}
Our main contributions can be summarized as follows:

\squishlist
\item {\bf Distributional Individual Fairness for Clustering:} We introduce distributional individual fairness for $\ell_p$-norm clustering problems using a general family of divergence functions.

\item  {\bf Approximation algorithms for Individually Fair Clustering:} We provide a generic solution template that adapts any algorithm for $\ell_p$-norm clustering objective to an individually fair solution. In particular, we give an algorithm for the individually fair $\ell_p$-norm $k$-clustering problem that achieves a constant factor approximation guarantee (Theorem~\ref{thm:individual_fairness}).  

\item {\bf Algorithms for Combined Fairness:} We show connections between individual fairness and group fairness, and extend our solution to combine the two paradigms. One interesting aspect of this result is that we enforce individual fairness only among the individuals belonging to the same protected group. We justify this relaxation in~\Cref{appendix:IF_SF} by demonstrating that the more stringent requirement of individual fairness across every pair of points can lead to trivial and expensive solutions. Our framework can be seamlessly combined with ideas developed in~\cite{bera19} to give a constant factor approximation algorithm that guarantees both group fairness (in expectation) and individual fairness among members of the same group (Theorem~\ref{thm:combined_fairness}).

% Further, similar in spirit to Dwork et al~\cite{dwork2012fairness}, we show that when the protected groups have statistically similar representation in the population, individual fairness implies group fairness (Theorem~\ref{}). 
\squishend
% \item We complement our positive results by proving that clustering under individual fairness constraints is NP-hard for the special case of $k$-center.(\syam{With the correct definition of $k$-center, this becomes trivial now. I guess we can just mention this.})
%\squishend
We provide extensive empirical evidence to support the effectiveness of our method.\footnote{We are contributing our code to the community.} Experiments show that our method achieves objective cost much better than predicted by our theoretical analysis while respecting individual fairness. Our solution is probabilistic. A single realization according to the distribution that our algorithm produces, might still be unfair to a pair of similar individuals. However, when the clustering algorithm is used upon repeated trials (e.g., profiling a customer for a sequence of different product recommendations), they would be assigned to the clusters with similar empirical distributions. This is the scenario our solution focuses on and tries to address.

\subsection{Related Work}
\label{sec:related}

Fairness in machine learning is a fast-evolving topic --- see~\cite{mehrabi2019survey} 
for a comprehensive survey of recent advances in this area. Our work mainly concerns with individual
fairness, a concept introduced by Dwork~\etal~\cite{dwork2012fairness}. Subsequently
in~\cite{zemel2013learning,lahoti2019ifair,lahoti13operationalizing}, the authors proposed methodologies to 
learn the similarity measure in order to achieve individual fairness. ~\cite{biega2018equity,speicher2018unified,kamishima2011fairness} 
also explored the direction of implicitly
learning the similarity measure in the context of ranking and classification problems. 
The approach of combining individual fairness and group fairness has been initiated in~\cite{dwork2012fairness} and further explored in~\cite{lahoti13operationalizing,speicher2018unified}.
However, none of these works consider the important case of clustering.

For clustering problems, in a seminal work, Chierichetti~\etal~\cite{CKLV18} initiated the study of fairness.
Their notion of fairness is defined at a group level ---
the population is partitioned into two protected groups and each group 
required to be {\em well-represented} in each cluster.
Subsequently, this notion has been greatly generalized to include more than 
two protected groups~\cite{RS18,Bercea2018,bera19,Ahmadian:2019}, and the groups are even allowed to be overlapping~\cite{bera19}.
The fairness notion advocated by these works operate within the ambit of
{\em disparate impact} doctrine~\cite{feldman2015certifying} --- each protected group
must be {\em almost} equally represented in the outcome of any algorithm. 
\cite{schmidt2018fair,Backurs2019,huang2019coresets} focused on designing scalable algorithms achieving group fairness.
Few other notions of fairness have been considered in the clustering domain such as
{\em proportionally} fair clustering~\cite{chen2019proportionally}, 
fair selection of cluster centers~\cite{Kleindessner2019,chiplunkar2020solve} and fair spectral clustering~\cite{Kleindessner2019spectral}.
None of these works address the question of
individual fairness and are orthogonal to the direction 
we take in this paper. Recently,~\cite{jung_center,mahabadi2020individual} consider a notion of individual fairness which requires every point $j$ to have a center within a distance of $r_j$ where $r_j$ is the minimum radius ball centered at $j$ that contains at least $n/k$ points. Our notion of individual fairness differs significantly from this notion and is not directly comparable. However, in our experiments, we consider a fairness similarity measure inspired by these works.

% For the clustering problem with group fairness constraints, Bera~\etal~\cite{bera19}
% gave a generic algorithm for converting an unfair set of clusters into fair clusters
% for any $\ell_p$-norm clustering objective with provable guarantees. 
% For specific clustering objectives,
% similar results are also obtained by Bercea~\etal~\cite{Bercea2018} and
% Ahmadian~etal~\cite{Ahmadian:2019}. Another line of work focused on designing scalable algorithms by considering
% coresets with group fairness constraints~\cite{schmidt2018fair,Backurs2019,huang2019coresets}

% Few other notions of fairness has been considered in the clustering domain such as
% {\em proportionally} fair clustering~\cite{chen2019proportionally}, 
% fair selection of cluster centers~\cite{Kleindessner2019}.
% Group fairness has been studied in the context of spectral clustering as well~\cite{Kleindessner2019spectral}.
% None of these work address the question of
% individual fairness and are orthogonal to the direction 
% we take in this paper.

\section{Problem Definitions and Preliminaries}
\label{sec:prelim}

We begin with the definition of statistical similarity between two distributions used in formulating individual fairness in clustering. 
\begin{definition}[$f$-divergence]
\label{def:f-div}
Let $P,Q$ be two probability measures on a discrete space $\calX$. Then for any function $f:[0,\infty)\rightarrow \Real$, where $f$ is strictly convex at 1 and $f(1)=0$, the $f$-divergence between $P$ and $Q$ is defined as $D_f(P||Q) = \sum_{x\in \calX}f(\frac{P(x)}{Q(x)})Q(x) $
\end{definition}
The above definition requires the following two assumptions for completeness:
\begin{inparaenum}[\bfseries (1)]
    \item $0\cdot f(\frac{0}{0}) = 0$ \,,
    \item $0\cdot f(\frac{a}{0}) = \lim_{x\to 0^{+}} xf(\frac{a}{x}) $ \,.
\end{inparaenum}
Some popular instances of $f$-divergence include total variation distance $D_{\TV}$ $\left( f(t)=\frac{1}{2}|t-1|\right)$ and  $\KL$-divergence $\left( f(t)=t\log t\right)$.

Next, we define various clustering problems that we shall consider in subsequent sections. Let $V$ be a set of points embedded in some metric space $(\cX,d)$. We use $[n]$ to denote the set $\{1,2,\cdots n\}$.
\begin{definition}[\vkp]
	\label{def:vanilla}
	 The  \vkp asks for 
	 \begin{inparaenum}[\bfseries (1)]
	 \item  a set of cluster centers ${\calC} \subseteq V$ of size at most $k$  and 
	 \item an assignment $\phi: V\rightarrow \calC$ of every point in $V$ to a center in ${\calC}$.
	 \end{inparaenum}
	 The objective is to minimize the $\ell_p$-norm distance, $\calL_p(\phi, \calC)=\left( \sum_{j\in V} d(j,\phi(j))^p\right)^{{1}/{p}}$.
\end{definition}
Some of the much-studied special cases are $k$-center $(p=\infty)$, $k$-median $(p=1)$, and $k$-means ($p=2$). Note that, for vanilla clustering, the assignment $\phi$ maps each point in $V$ to its closest center in ${\cC}$ and hence fully determined by ${\cC}$. We next define the individually fair clustering problem. Let $\fairness:V\times V \rightarrow \Real_{\geq 0}$ be a non-negative  fair similarity measure defined over all pair of points in $V$. Note that $\fairness$ may not be a metric. 
 
\begin{definition}[\kfp]
\label{def:individual_fairness}
Assume we are given a function $f$ as in Definition~\ref{def:f-div}. Then, \kfp asks for
\begin{inparaenum}[\bfseries (1)]
\item a set of cluster centers ${\calC} \subseteq V$ of size at most $k$ and
\item  a distribution $\mu_j$ over $\calC$ for each point $j \in V$,
\end{inparaenum}
% \begin{enumerate}
%     \item  a set of cluster centers ${\calC} = \{c_1,c_2,\ldots,c_k\}$ where each $c_k\in V$ and
%     \item a distribution $\mu_j$ over $\calC$ for each point $j \in V$,
% \end{enumerate}
such that 
\begin{align}
\label{eq:ind-fair}
D_f(\mu_{j_1} || \mu_{j_2}) \leq \fairness(j_1, j_2), \forall j_1,j_2\in V
\end{align}
 The objective is to minimize 
 $   \calL_p(\mu, \calC) := \left( \sum_{j\in V} \Ex_{c{\sim}\mu_j}(d(j,c)^p) \right)^\frac{1}{p}$.
% $     = \left( \sum_{j\in V} \sum_{i\in [k]} \mu_j(i)d(j,i)^p \right)^\frac{1}{p}.
% $
\end{definition}
The definition of individually fair $k$-center is not precisely captured by the above definition. We treat that separately in Appendix~\ref{appendix:3}. We denote the optimal cost of any instance $\calI$ of \vkp as $\vopt(\calI)$ and  that of any instance $\calJ$ of \kfp as $\iopt(\calJ)$. 

\begin{comment}
 
It is quite straightforward to formulate a mixed integer linear program for the above formulation. For a fixed point $j\in V$, $x_{ij}, i\in V$ denotes the probability that $j$ is assigned to the cluster center $i$ in the solution. The variable $y_i\in \{0,1\}$ indicates whether the a cluster center is present at the point $i$. Further assume $\vec{x_j}$ denote the distribution corresponding to the fixed point $j\in V$.

\begin{align}
    &\min \sum_{i,j} x_{ij} d(i,j)^p \\
    \text{s.t.}~ & \sum_{i\in V} x_{ij} = 1 ~~\forall j \in V \,, \\
    & x_{ij} \leq y_i ~~\forall i,j \in V \\
    & \sum_{i\in V} y_i \leq k \\
    & D(\vec{x_{j_1}},\vec{x_{j_2}}) \leq d(j_1,j_2) ~~\forall j_1,j_2 \in V \\
    & y_i \in \{0,1\} \\
    & x_{ij} \in [0,1]
\end{align}

\end{comment}
We now define a problem that ensures both statistical and individual fairness. Note that, in this definition,
we only enforce individual fairness among individuals that belong to the same protected group (see~\Cref{appendix:IF_SF}).
\begin{definition}[\total]
\label{def:stat_indv_fairness}
Assume we are give an instance of the \kfp problem. Additionally, we are given $\ell$-many 
(possibly overlapping) protected groups $G_1,G_2,\ldots,G_{\ell}$
and for each such group we are given two input {\em group} fairness parameters $\alpha_i$ and $\beta_i$.
The goal and the objective remain the same. 
The output distributions $\mu_j, \forall j\in V$ must satisfy the following two constraints.
\begin{enumerate}
    \item For each cluster, the expected fraction of the points from group $G_i$ lies between $\beta_i$ and $\alpha_i$,
    \item $D_f(\mu_{j_1} || \mu_{j_2}) \leq \fairness(j_1,j_2)$ for each pair of points $j_1,j_2 \in G_p$, for all $p\in [\ell]$.
\end{enumerate}
\end{definition}

We remark here that there exists a trivial and potentially very expensive feasible solution to both
the individual and combined fair clustering problems --- simply assign a uniform 
distribution to each point (for the combined fair clustering, this assumes that the instance is feasible with respect to group fairness parameters $\alpha$ and $\beta$). See~\Cref{appendix:IF_SF} for a discussion on the feasibility question.
% \begin{lemma}[Jensen's inequality
% ~\cite{jensen1906fonctions}]
% \label{lem:jensen}
% Let $\phi$ be a real valued convex function and $p$ be a distribution over finite discrete space $\cX$.
% Then, $\phi (\sum_{i\in \cX} p_ix_i) \leq \sum_{i \in \cX} p_i \phi(x_i)$.
% \end{lemma}

\section{Algorithm for Individually Fair Clustering }
\label{sec:alg-if}
In this section, we present our main theoretical result. 
We give an algorithmic framework for solving the individually fair clustering problem (\Cref{alg:kfp}).~\Cref{thm:individual_fairness}
captures its theoretical guarantees. 

\begin{algorithm}[!ht]
    \caption{{\bfseries \ALG$(\calI)$} --- Algorithm for \kfp}
    \label{alg:kfp}
    \begin{algorithmic}[1]
        % \STATE {\bfseries \ALG$(\calI)$} 
        \STATE Run a $\rho$-approximation algorithm for \vkp on $\calI$ --- let $\calC$ be the set of centers.
        \STATE Solve the \fairopt~problem on instance $\calJ = (V,\calC,f,\fairness)$ --- let ${\mu}$ be the solution. 
        \STATE  return $(\calC,{\mu})$
    \end{algorithmic}
\end{algorithm}
Suppose we are given an instance $\calI=(V,d,f,\fairness)$ for \kfp. 
We first disregard $f$ and $\fairness$, and use any existing algorithm for the \vkp problem to obtain a set of cluster centers $\calC$.
We then create a constrained optimization problem \fairopt on the instance $\calJ = (V, \calC,f,\fairness)$, 
as given in~\Cref{eqn:FairLP,eqn:sum,eq:ind-fair-LP,eqn:fraction}, and solve it. 
We combine the solution of both the steps and return it as
our final output.
\begin{align}
\label{eqn:FairLP}
    \fairopt(\calJ):&\min \sum_{j\in V}\sum_{c\in \calC} x_{cj} d(c,j)^p \\ %\tag{Fair-OPT}
    \text{s.t.}~ & \sum_{c\in \calC} x_{cj} = 1 ~~\forall j \in V \, \label{eqn:sum}\\
    & D_{f}(\vec{x}_{j_1} || \vec{x}_{j_2}) \leq \fairness(j_1,j_2) ~~\forall j_1,j_2 \in V 
    \label{eq:ind-fair-LP}\\
    & 0\leq  x_{cj} \leq 1 \label{eqn:fraction}
\end{align}

We now discuss the \fairopt~problem. For each $j\in V$ and $c\in \calC$, let $x_{cj}$ be the probability
that the client $j$ is assigned to the center $c$. Hence, $\Vec{x_j}$ will give the desired 
distribution $\mu_j$ corresponding to $j$ over the set of centers $\calC$. 
The first constraint ensures that each client is assigned a distribution and the second
one enforces the individual fairness constraints~\eqref{eq:ind-fair}. 
Clearly, any solution to \fairopt~ is also a feasible solution to \kfp.

Note that the computational complexity of solving the above constrained optimization depends on the constraints~\eqref{eq:ind-fair-LP}. For example, if the LHS of these constraints are convex functions of $x$, then we can solve this in polynomial time. Indeed, that is the case for many common choices of $D_f$ ($D_{\TV}$, KL-divergence, etc.). Let $\cA_1$ be a $\rho$-approximate algorithm for 
\vkp with running time $T(\cA_1)$ and $\cA_2$ be an optimal solver for the 
\fairopt~problem with running time $T(\cA_2)$. Then, our main result is the following theorem.
\begin{theorem}
	\label{thm:individual_fairness}
	Given an instance $\calI$ to \kfp, let $(\calC,\phi)$ be a $\rho$-approximate solution of \vkp on $\calI$. 
	%Then, there exists an algorithm for an \ifa instance on $V, \calC$ 
	Then,~\Cref{alg:kfp} produces distributions $\mu_j, \forall j\in V$, such that $\calL_p(\mu, \calC) \leq 3^{(1 - \frac{1}{p})}(\rho+2)\cdot \iopt(\calI)$ and it runs in time $O(T({\cA_1}) + T(\cA_2) )$.
\end{theorem}

In the remainder of this section, we prove Theorem~\ref{thm:individual_fairness}. We state and use several lemmas in this section whose proofs we defer to the Appendix~\ref{appendix:3}. 
%We now prove the upper bound on the expected cost of our solution as claimed in Theorem~\ref{thm:individual_fairness}. 
We emphasize that the cost guarantee of our
algorithm is with respect to $\iopt(\calI)$  and {\em not}
with respect to $\vopt(\calI)$. It is indeed possible that 
$\iopt(\calI)$ is much larger than $\vopt(\calI)$, and hence
the clustering cost of our algorithm could be much larger compared 
to $\vopt(\calI)$. The cost of achieving fairness depends on the
fairness measure $\fairness$ and we discuss it in the experiment section (\Cref{sec:exp}).

Assume $(\calC^{\star},x^{\star})$ is an optimal solution to instance $\calI$ of \kfp
and \ALG$(\calI)$ returns $(C,\mu)$.
We construct a feasible solution $x$ to \fairopt($\calJ=(V,C,f,\fairness)$) using $\calC^{\star}$ and $x^{\star}$.
\ALG outputs the optimal solution to \fairopt, hence,
$\calL_p(\mu, \calC) \leq \calL_p(x, \calC)$. So, to prove the theorem, it is sufficient to 
bound $\calL_p(x, \calC)$.

Let $\phi:\calC^{\star}
\rightarrow \calC$ be a function that maps each center in $\calC^{\star}$ to its closest center in $\calC$:
$\phi(c^{\star}) = \arg\min_{c\in \calC} d(c,c^{\star})$, breaking ties arbitrarily.
Let $\phi^{-1}(c)$ denote the set of centers mapped to $c\in \calC$: $\phi^{-1}(c) = \{c^{\star}\in \calC^{\star}: \phi(c^{\star})=c\}$.
Note that $\phi^{-1}(c)$ can be empty for some $c\in \calC$.
For each $j\in V$ and each $c\in \calC$, 
set $x_{cj} = \sum_{c^{\star} \in \phi^{-1}(c) }x^{\star}_{c^{\star}j}$. In words,
for a fixed point $j\in V$ and a fixed center $c\in \calC$, we look at the centers 
in the optimal solution that are mapped to $c$ by $\phi$, and sum the 
corresponding probabilities to get $x_{cj}$.

We first claim the following structural property of the mapping $\phi$. This claim bounds the distance between a point $j \in V$ and a center $c\in {\cC}$ in terms 
of the distance between $j$ and its closest center in ${\cC}$ and the distance between
$j$ and any optimal center $c^{\star}$ that is mapped to $c$ by $\phi$.
\begin{restatable}{claim}{distanceClaim}
% \begin{claim}
\label{clm:distance}
Assume $c\in \calC$ be a center such that $\phi^{-1}(c)$ is non-empty.
For a point $j\in V$, let $c_j$ be its closest center in $\calC$: $c_j = \arg \min \{d(j,c):c\in \calC\}$. Then, 
%there exists a constant $\phi=\psi(p)$, such that
for each $c^{\star} \in \phi^{-1}(c)$ and for each $j\in V$, we have 
\begin{align*}
    d(j,c)^p \leq 3^{p-1} \left(2d(j,c^{\star})^p +  d(j,c_j)^p\right) 
\end{align*}
% \end{claim}
\end{restatable}
% \begin{proof}
% We defer the proof to the ~\Cref{appendi:3}.
% \end{proof}

%We defer the proof of~\Cref{clm:distance} to ~\Cref{appendix:3}. 
Using~\Cref{clm:distance}, we show that $x$ is a {\em low cost} 
solution to the \fairopt($\calJ$) problem in Lemma~\ref{lem:feasibility}.
Theorem~\ref{thm:individual_fairness} then follows immediately from Lemma~\ref{lem:feasibility}. %We give the proof of~\Cref{lem:feasibility} in~\Cref{appendix:3}.
% Due to lack of space, we defer the proofs to the appendix; instead we give a high level idea of the proof here.
\begin{restatable}{lemma}{feasibleLemma}
%\begin{lemma}
\label{lem:feasibility}
$x$ is a feasible solution to \fairopt($\calJ$) with cost $\calL_p(x, \calC) \leq 3^{\left( 1-\frac{1}{p} \right) } \left( \rho+2 \right)\cdot \iopt(\calI)$.
%\end{lemma}
\end{restatable}

\begin{remark}
\label{remark:k-center}
The individually fair $k$-center ($p=\infty$) problem is not handled directly by~\Cref{alg:kfp}.
In particular, stating the \fairopt~optimization problem (\Cref{eqn:FairLP}) with $p=\infty$ requires the standard technique of ``guess the optimal value''. See Appendix~\ref{appendix:3} for details.
\end{remark}
% \begin{restatable}{lemma}{costLemma}
% % \begin{lemma}
% \label{lem:cost}
% $\calL_p(x, \calC) \leq 3^{\left( 1-\frac{1}{p} \right) } \left( \rho+2 \right)\cdot \iopt(\calI)$.
% % \end{lemma}
% \end{restatable}
% \begin{proof}
% See~\Cref{appendi:3}.
% \end{proof}

% We complement the above results by proving in the appendix that \kfp is NP-hard, even for the special case of $\ell_{\infty}$-norm, which is the classical $k$-center problem and $D_{\TV}$ as the choice for $D_f$. Note that the above hardness does not directly follow from the NP-hardness of the corresponding hard assignment problems because, in \kfp, we are only required to find a distribution of each point on a set of $k$ cluster centers.

%\subsection{On the Fairness Metric}

%In our discussions so far, we use the same metric space $d$ for the objective function as well as 
% individual fairness measure. In other words, to provide equal
%treatment, we have assumed points close to each other in the 
%metric space as similar. However, often in practice, one might have access to side information about the similarity measure between individuals that are independent of the feature space and it is more natural to define fairness according to that measure.
%We show that our algorithmic framework is, in fact, capable of handling a different similarity measure for fairness with same theoretical guarantees. We present the formal details in~\Cref{appendix:separate}.

%\input{source/relation.tex}

\section{Individual Fairness and Group Fairness}
\label{sec:IF_SF}
In this section, we consider the \total problem.
At a high level, our algorithmic strategy
remains the same --- we first solve the \vkp to find the cluster centers,
and then solve a suitable constrained optimization program to find the distribution corresponding to each point. We describe in~\Cref{appendix:IF_SF} the constrained optimization problem analogous to the \fairopt~problem given in ~\Cref{sec:alg-if}.
% \begin{align}
% \label{eqn:total_fair_assgn}
%     \min& \sum_{j\in V}\sum_{c\in \calC} x_{cj} d(i,j)^p  \\
%     \text{s.t.}~ & \sum_{c\in \calC} x_{cj} = 1 ~~\forall j \in V \, \nonumber\\
%     & D_{f}(\vec{x_{j_1}} || \vec{x_{j_2}}) \leq d(j_1,j_2) ~~\forall r \in [\ell] , j_1,j_2 \in G_r 
%      \nonumber\\
%     & \beta_r \sum_{j\in V} x_{cj} \leq  \sum_{j\in G_r} x_{cj} \leq \alpha_r \sum_{j\in V} x_{cj}, \forall c\in \calC\,, r \in [\ell]  \nonumber\\ %\label{eq:group-fair-LP2}
%     & 0\leq  x_{ij} \leq 1 \nonumber
% \end{align}
Reusing notation, assume $\iopt(\calI)$ denote
the optimal cost of the instance $\calI$. We then prove the following theorem in~\Cref{appendix:IF_SF}.
\begin{restatable}{theorem}{IFSF}
	\label{thm:combined_fairness}
	Given an instance $\calI$ to \total, let $\calC$ be a $\rho$-approximate solution for the corresponding \vkp on $\calI$. Then, there exists an algorithm which produces feasible distributions $\mu_j, \forall j\in V$, such that $\calL_p(\mu, \calC) \leq 3^{1-\frac{1}{p}}(\rho+2)\cdot \iopt(\calI)$.
\end{restatable}

Note that, the case of $p=\infty$ ($k$-center) requires special case (see~\Cref{remark:k-center}) --- we handle this in~\Cref{appendix:IF_SF}.
Finally, we consider the special case of $\fairness=d$, that is when the fairness 
similarity measure is given by the underlying distance metric, and observe the conditions under which individually fair clustering solutions guarantees group fairness. Our characterization is similar to the one discussed in the work of~\cite{dwork2012fairness} and given in~\Cref{appendix:relation}.

\section{Experimental Evaluation}
\label{sec:exp}

% About the simulation environment
In this section, we present extensive empirical evaluations of our algorithms. We implement our algorithms in Python 3.6 and simulate on Intel Xeon CPU E5-2670 v2 @ 2.50GHz 20 cores and 96 GB 1333 MHz DDR3 memory. We use IBM CPLEX for solving linear programs. \footnote{\href{https://github.com/nihesh/distributional\_individual\_fairness\_in\_clustering}{https://github.com/nihesh/distributional\_individual\_fairness\_in\_clustering}}

% Summary of results
Although our algorithmic framework can handle any $\ell_p$-norm based objective, we focus on the widely popular $k$-means clustering for demonstration. We measure individual 
fairness against {\em total variation norm}, $D_{\TV}(\mu_x || \mu_y) = \frac{1}{2}\sum_{c\in \cC}|\mu_x(c) - \mu_{y}(c)|$, a widely used $f$-divergence measure.
Based on our experiments, we report the following key findings. 
\begin{inparaenum}[\bfseries (1)]
    \item Variants of $k$-means and other clustering algorithms that guarantee group fairness are largely unfair to individuals.
    \item Our algorithms provide individual fairness by paying at most $1.08$ times more than the optimal cost.
    \item Unlike group fairness, individual fairness comes at a higher cost when compared against vanilla $k$-means.
\end{inparaenum}

% Key findings

% Brief note on datasets used
\noindent
{\bf Datasets.} We use five datasets from UCI Machine Learning Repository \cite{uci-repo}. 
\footnote{\href{https://archive.ics.uci.edu/ml/datasets/}{https://archive.ics.uci.edu/ml/datasets/}}
\begin{inparaenum}[\bfseries (1)]
\item Bank - 4,521 points \cite{bank-dataset}
\item Adult - 32,561 points \cite{adult-dataset}
\item Creditcard - 30,000 points \cite{creditcard-dataset} 
\item Census1990 - 2,458,285 points \cite{census1990-dataset}
\item Diabetes - 101,766 points \cite{diabetes-dataset}. 
\end{inparaenum} We remark that most of the previous works on fairness in clustering~\cite{CKLV18,bera19,Backurs2019,huang2019coresets} 
focused on these datasets.

% Define the algorithms used in experiments
\noindent
{\bf Algorithms.} We use Lloyd's algorithm \cite{scikit-learn} to solve vanilla $k$-means and approximate the centers by its nearest neighbour in $V$. $\hkm$ denotes hard $k$-means (binary assignment of points to centers) and $\skm$ denotes soft $k$-means~\cite{skm-lecture-notes, DudaHS01}. $\skm$ outputs a set of $k$ centers $\{c_1, c_2, \cdots c_k\}$ and for a fixed stiffness parameter $\beta$, assigns $x\in V$ to a center $c$ with probability $\frac{ e^{-\beta d(c,x)^ 2} } { \sum_{\ell = 1}^{k} e^{-\beta d(c_{\ell},x)^ 2} }$. $\algif$ denotes the algorithm for individual fairness from Section~\ref{sec:alg-if} and $\algtf$ denotes the algorithm for combined fairness from Section~\ref{sec:IF_SF}. 
$\alggf$ denotes the algorithm for group fairness from~\cite{bera19}.
$\optif$ and $\opttf$ denote the optimal solution to the natural LP relaxation (allowing the fractional opening of centers) for \kfp and \total respectively. They provide lower bounds to the cost of the optimal solution of the corresponding problems.

% Description of how individual fairness constraints are enforced in practice.
\noindent
{\bf Fairness Similarity Measures.} 
We consider two different fairness similarity measures $\fairness_1$ and $\fairness_2$. Both the measures are defined using the underlying distance metric $d$ in the given feature space. We choose $\fairness_1 = d$, scaled linearly so that $\fairness_1(j_1, j_2) \in [0, 1]$ $\forall j_1, j_2 \in V$.  
In order to lower the computational requirement, we enforce $\fairness_1$ constraints only between every $i\in V$ and its $m$ nearest neighbors. $\fairness_2$ is defined in a more local way. For each $i\in V$, we consider the smallest ball $B_i$ of radius $r_i$ centered at $i$, such that $B_i$ contains at least $\floor{{|V|}/{k}}$ points. Then, we define $\fairness_2(i,j) = {d(i,j)}/{r_i}, \forall j\in B_i$ and $\fairness_2(i,j) = 1$, otherwise. The motivation behind $\fairness_2$ is 
inspired by the individual fairness notion in~\cite{jung_center,mahabadi2020individual}. More specifically, in $\fairness_2$, each point is required to be treated similarly to its closest $\floor{{|V|}/{k}}$ neighbours. For \tf, we enforce $\fairness_1$ and $\fairness_2$ only within protected groups. 

{\bf Implementation Details.}
% Common simulation parameters used over all the experiments
We subsample the datasets to 1000 points selected uniformly at random and run the experiments on a subset of numerical attributes. 
The numerical attributes are normalized to zero mean and unit variance. We choose two protected attributes for each dataset, set $\delta = 0.2$ (measure of tightness of group fairness constraints, introduced in~\cite{bera19}) and set $m = 250$. We run the algorithms for $k = 2, 4, 6, 8, 10$. This configuration of parameters is used in all the simulations unless mentioned otherwise.

Due to space constraints, we present a subset of our results here --- further results, including runtime of our algorithm on variable dataset sizes, are given in  Appendix~\ref{supp:exp}.

\begin{table}[ht]
\centering
\caption{Percentage of individual fairness constraint violations of $\skm$ when $\skm$ and $\algif$ incur the same clustering cost.}
\vspace{1mm}
\subfloat[][Fairness similarity $\calF_1$]{
     \begin{tabular}{ c c c c c }
     \label{tab:skmviolations_global}
         Clusters ($k$) & $4$ & $6$ & $8$ & $10$ \\
         \hline
         Adult   & 88 & 94 & 98 & 99 \\
         Creditcard  & 61 & 76 & 83 & 85 \\
         Census1990 & 25 & 34 & 44 & 50 \\
         \hline
    \end{tabular}
}
\qquad
\subfloat[][Fairness similarity $\calF_2$]{
    \begin{tabular}{ c c c c c }
    \label{tab:skmviolations_local}
        Clusters ($k$) & $4$ & $6$ & $8$ & $10$ \\
        \hline
        Adult  & 4 & 5 & 7 & 8 \\
        Creditcard & 6 & 5 & 6 & 6 \\
        Census1990  & 7 & 11 & 13 & 11 \\
        \hline
    \end{tabular}
}
\label{tab:skmviolations}
\end{table}

{\bf Unfairness of $\boldsymbol{\skm}$. }  In~\Cref{tab:skmviolations}, we demonstrate the unfairness of soft $k$-means. Note that the output of $\skm$ depends on the stiffness parameter $\beta$. We experimentally choose $\beta$ such that the cost of $\skm$ is equal to the cost of $\algif$.~\Cref{tab:skmviolations_global} shows the percentage of individual fairness constraints violated, with respect to $\fairness_1$ and~\Cref{tab:skmviolations_local} shows the same with $\fairness_2$. 
Observe that $\fairness_2$ is a much relaxed fairness measure compared to $\fairness_1$: for each point, similarity is measured locally, with respect to its $\floor{{|V|}/{k}}$ nearest neighbors. Even with such relaxations, $\skm$ exhibits unfair treatment of similar points. Our solution does not violate any individual fairness constraints.

\begin{figure}[ht]
    \centering
    \subfloat[][Fairness similarity $\calF_1$] {
        \includegraphics[width=0.45\linewidth]{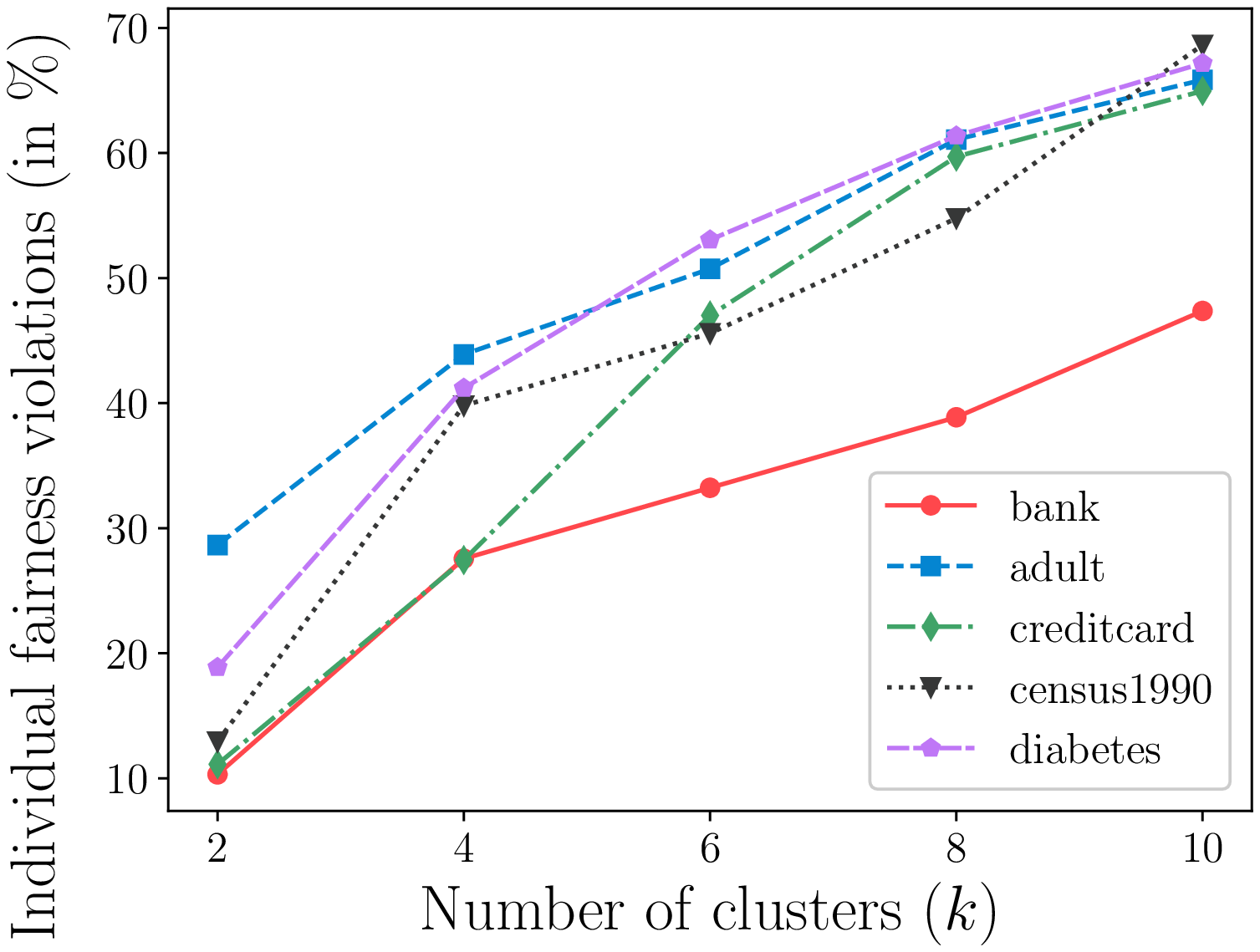}
        \label{fig:sf_global}
    }
    \qquad
    \subfloat[][Fairness similarity $\calF_2$] {
        \includegraphics[width=0.45\linewidth]{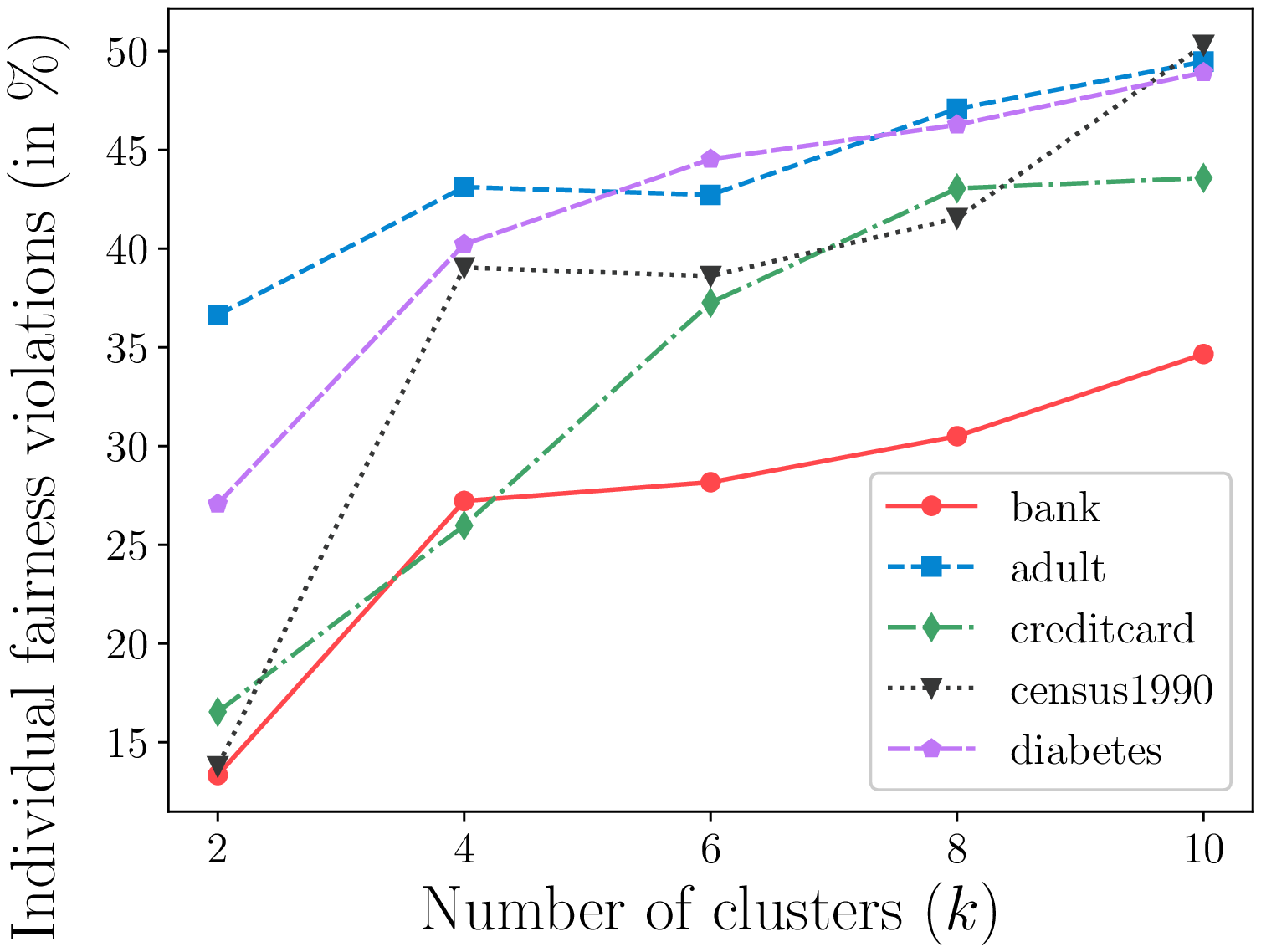}
        \label{fig:sf_local}
    }
    \caption{Percentage of individual fairness constraint violations of $\alggf$ vs number of clusters}
    \label{fig:sf_unfairness}
\end{figure}

{\bf Unfairness of $\boldsymbol{\alggf}$.} In~\Cref{fig:sf_unfairness}, we show that group fairness does not imply individual fairness. We observe the percentage of individual fairness constraints violated by $\alggf$ for different values of k and infer that, for $k \geq 4$, at least $25\%$ of the constraints are violated in the best case, and violations increase monotonically as k increases (as expected). 

\noindent
{\bf Cost Analysis of Our Algorithms.} In this section, we compare the cost of $\algif$ and $\algtf$ against $\optif$ and $\opttf$, respectively. Since $\optif$ and $\opttf$ are computationally expensive, we reduce the size of the dataset to $80$ points chosen uniformly at random, and set $m = 20$. We present the plots for two datasets here, and 
the rest are in Appendix~\ref{supp:exp} (similar trend).

\begin{figure}[ht]
    \centering
    \subfloat[][Fairness similarity $\calF_1$] {
        \includegraphics[width=0.45\linewidth]{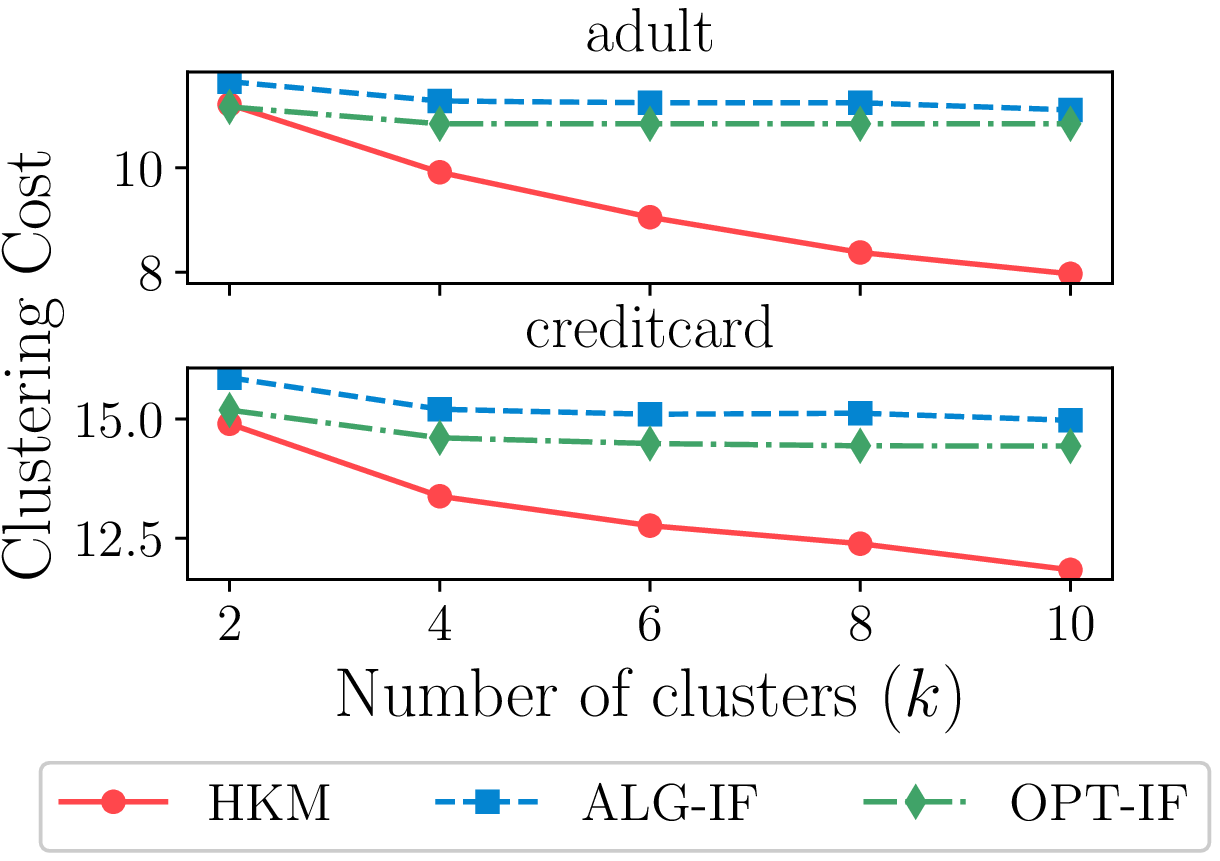}
        \label{fig:if_alg_opt_global}
    }
    \qquad
    \subfloat[][Fairness similarity $\calF_2$] {
        \includegraphics[width=0.45\linewidth]{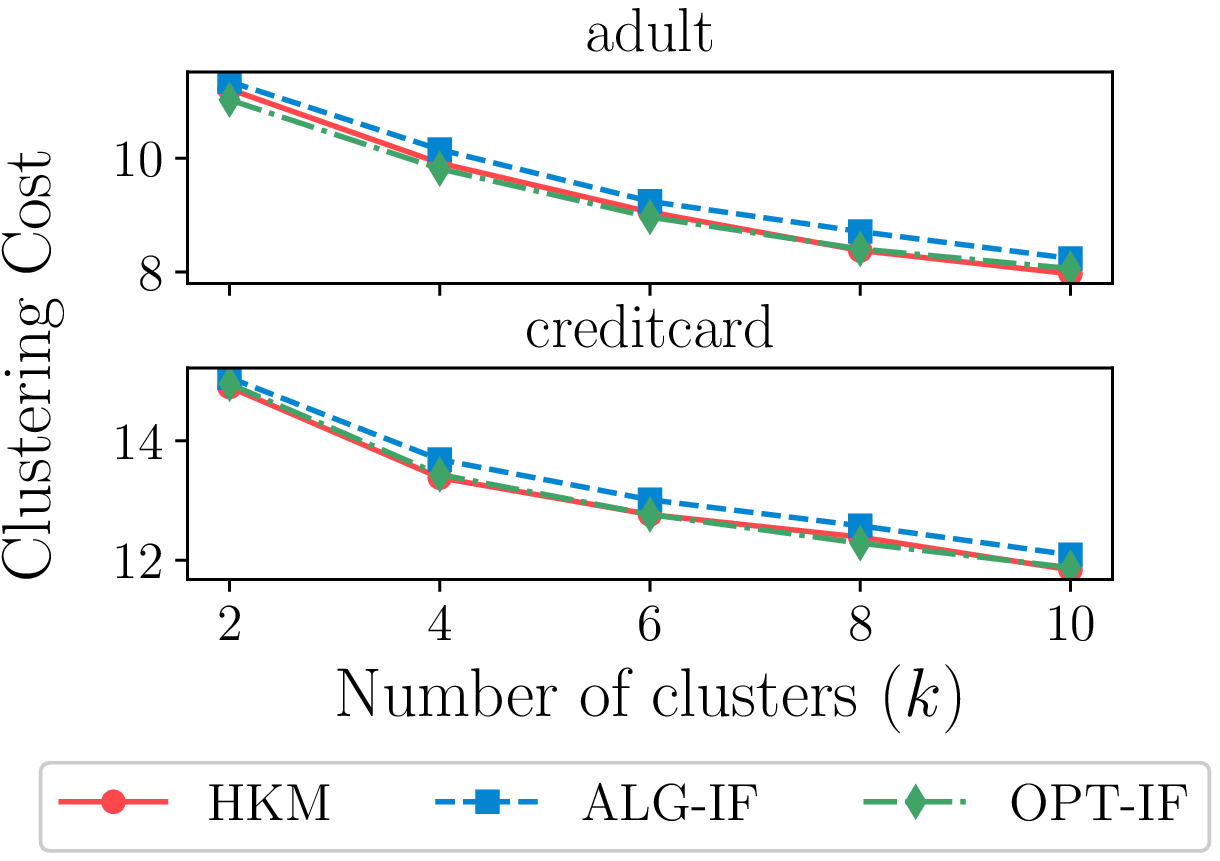}
        \label{fig:if_alg_opt_local}
    }
    \caption{Clustering cost vs number of clusters for $\algif$, $\optif$ and $\hkm$.}
    \label{fig:if_alg_opt}
\end{figure}

In~\Cref{fig:if_alg_opt}, we compare the cost of $\algif$ and $\optif$ using fairness similarity $\fairness_1$ and $\fairness_2$. We observe that the approximation ratio is at most $1.08$, which is significantly better than the bound given in~\Cref{thm:individual_fairness}. 

\begin{figure}[ht]
    \centering
    \subfloat[][Fairness similarity $\calF_1$] {
        \includegraphics[width=0.45\linewidth]{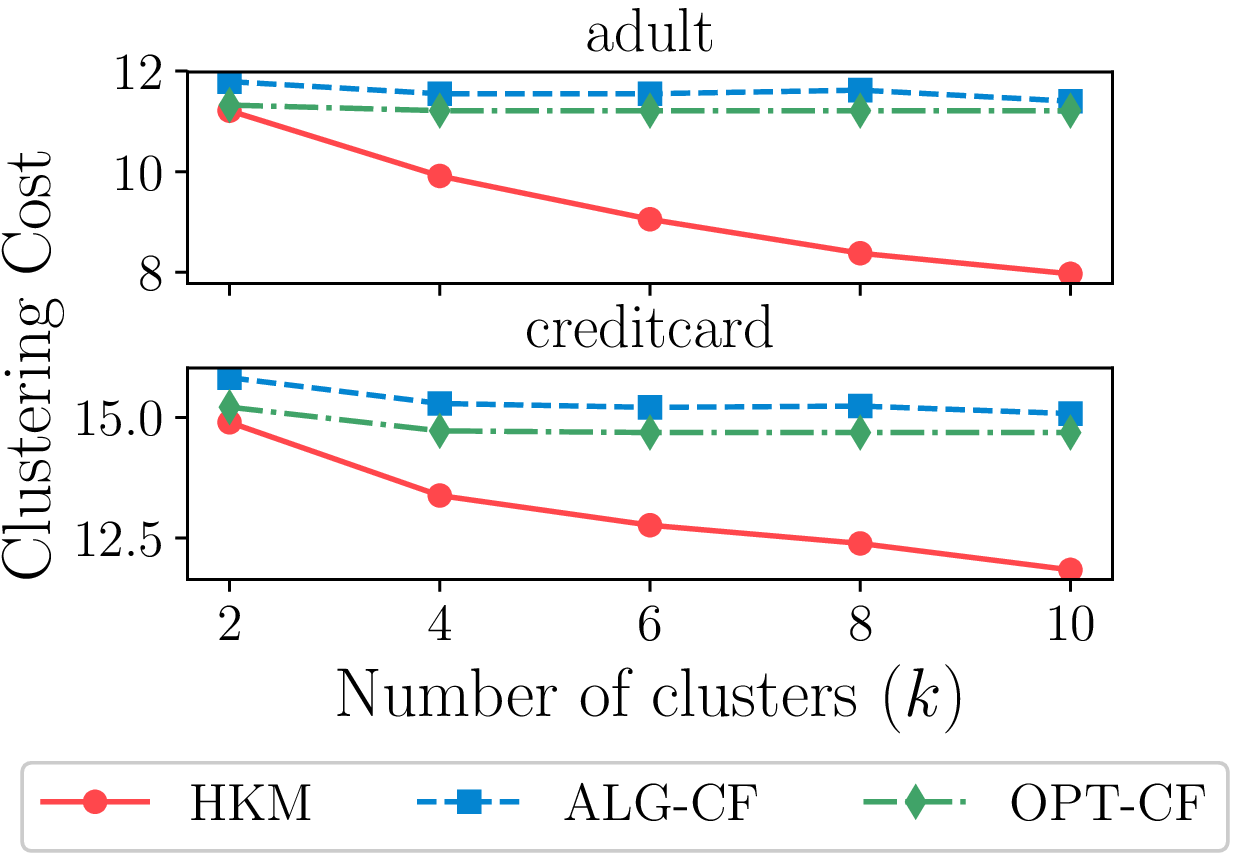}
        \label{fig:cf_alg_opt_global}
    }
    \qquad
    \subfloat[][Fairness similarity $\calF_2$] {
        \includegraphics[width=0.45\linewidth]{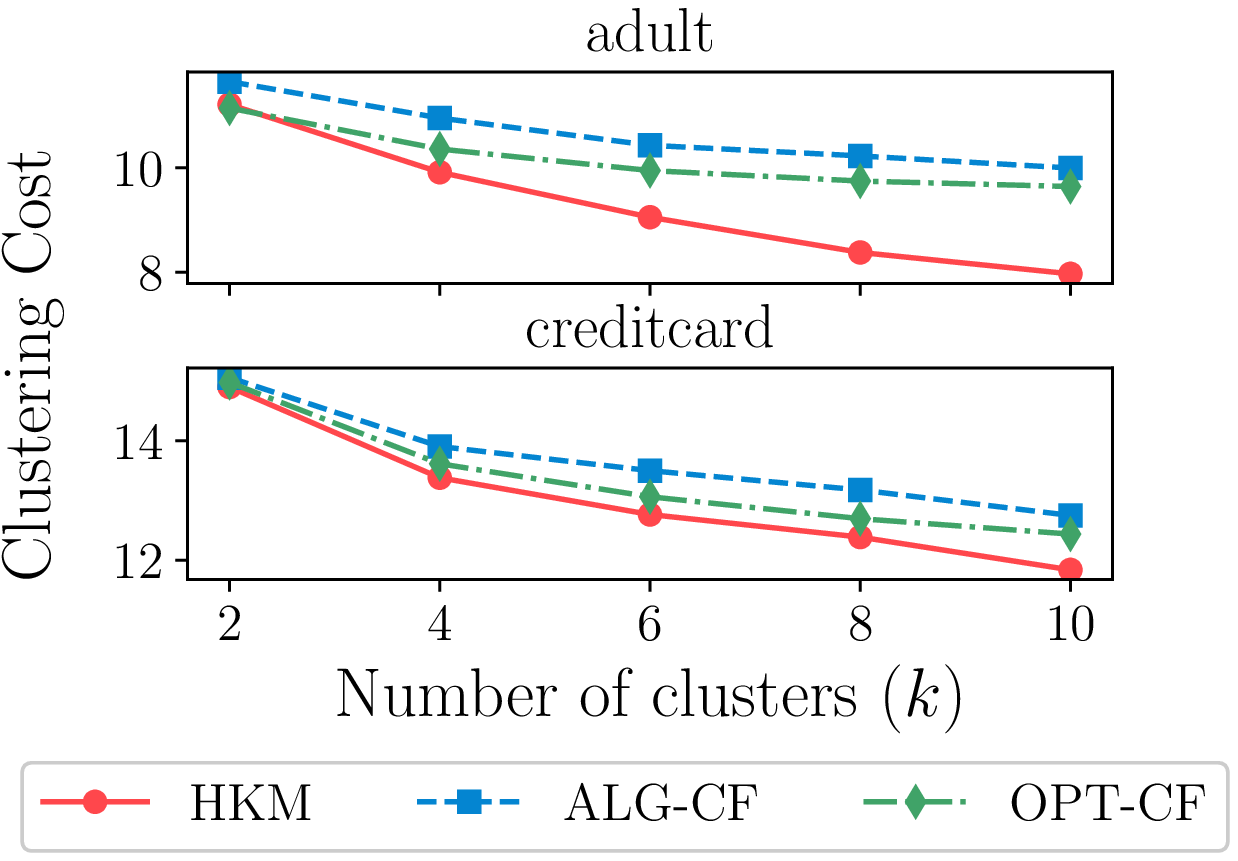}
        \label{fig:cf_alg_opt_local}
    }
    \caption{Clustering cost vs number of clusters for $\algtf$, $\opttf$ and $\hkm$.}
    \label{fig:cf_alg_opt}
\end{figure}

In~\Cref{fig:cf_alg_opt}, we compare the cost of $\algtf$ and $\opttf$ using fairness similarity $\fairness_1$ and $\fairness_2$. Similar to $\algif$, we observe that the approximation ratio is at most $1.06$, which is significantly better than the bound given in~\Cref{thm:combined_fairness}.

\noindent
{\bf Price of Individual Fairness.}~\Cref{fig:if_alg_opt} and~\Cref{fig:cf_alg_opt} shows that the cost of $\optif$ and $\opttf$ can be at most $1.5$ times larger as compared to $\hkm$.
In contrast, ~\cite{bera19} showed that group fairness can be achieved by paying at most 1.15 times the $\hkm$ cost
(for all the datasets). It suggests that individual fairness comes at a higher price. We elaborate on this further in Appendix~\ref{appendix:3}.

\section{Conclusion}

In this work, we initiate the study of individual fairness in clustering, inspired by the notion of Dwork~\etal~\cite{dwork2012fairness} in the context of classification. We discuss and demonstrate the limitations of group fairness alone.
We give a general framework for handling individual fairness and combined fairness
for a variety of clustering objectives as well as statistical distance measures. Empirically, we demonstrate the effectiveness of our approach. One caveat of our generic framework is that we rely on an efficient solver for a convex optimization problem. We leave the problem 
of designing more efficient and scalable algorithms for specific instances of 
$f$-divergence as an interesting future research direction.
\newpage
\bibliographystyle{plain}
\bibliography{refs}

% \newpage
\appendix
% \section{On the Cost of Individually Fair Clustering}
% \label{appendix:cost}

\section{Missing Details from~\Cref{sec:alg-if}}
\label{appendix:3}
In this section, we fill out the details of various items
that we have omitted in the main body due to lack of space. 
In~\Cref{appendix:proof}, we complete the proof of~\Cref{thm:individual_fairness}. In~\Cref{appendix:k_center}
we take up the case of $k$-center and discuss how to modify our
algorithm to get the same result as given in~\Cref{thm:individual_fairness}. Finally, in~\Cref{appendix:cost}, we discuss the {\em price} of achieving individual fairness by comparing
the cost of a fair clustering solution against the corresponding vanilla clustering solution.

\subsection{Proof of~\Cref{thm:individual_fairness}}
\label{appendix:proof}
In this section, we present the proofs of various lemmas and claims that
are used in proving ~\Cref{thm:individual_fairness}. We use Jensen's Inequality in the proof, and for the sake of completeness, we include it here.
\begin{lemma}[Jensen's inequality
~\cite{jensen1906fonctions}]
\label{lem:jensen}
Let $g$ be a real-valued convex function and $p$ be a distribution over finite discrete space $\cX$.
Then, $g (\sum_{i\in \cX} p_ix_i) \leq \sum_{i \in \cX} p_i g(x_i)$.
\end{lemma}
We now restate the claim regarding the structural property of the mapping $\phi$ and prove it.
\distanceClaim*
\begin{proof}
We begin the proof by first considering $p=1$. In this case, $\psi=1$.
This was implicitly proved in~\cite{bera19}.
For completeness, we present a proof here as well.
The proof follows by application of triangle inequality and definition of $\phi$.
\begin{align}
%\label{eqn:proof}
    &d(j,c) \nonumber \\
    &\leq d(j,c^{\star}) + d(c^{\star},c) \,,&\text{(triangle inequality)} \nonumber \\
    & \leq d(j,c^{\star}) + d(c^{\star},c_j) \,,&\text{(since $c^{\star} \in \phi^{-1}(c)$)} \nonumber \\
    & \leq d(j,c^{\star}) + d(c^{\star},j) + d(j,c_j) \,, &\text{(triangle inequality)}\nonumber \\
    & = 2 d(j,c^{\star}) + d(j,c_j) \,. &  \label{eq:abc}   
\end{align} 
To prove for general $p$ such that $p\geq 1$, we use the convexity of the function $h(x)=x^p$.
In particular, applying Jensen inequality, we derive,
\begin{align*}
     &d(j,c) ^p \\ 
     &\leq (2 d(j,c^{\star}) + d(j,c_j))^p\,, ~~\text{(from~\Cref{eq:abc})} \\
     & \leq 3^{p-1} (2 d(j,c^{\star})^p + d(j,c_j)^p) ~~\text{(Jensen's inequality)} \,.
\end{align*}
This completes the proof of the claim.
\end{proof}

We now prove the main technical lemma (~\Cref{lem:feasibility}).
\feasibleLemma*
\begin{proof}
We first prove that $x$ is a feasible solution to \fairopt($\calJ$). 
First, we show that $\vec{x_j}$ is a 
probability distribution. Clearly, $0 \leq x_{cj} \leq 1$ for all $c$:
\begin{align*}
x_{cj} = \sum_{c^{\star} \in \phi^{-1}(c)}x^{\star}_{c^{\star}j} \leq \sum_{c^{\star}\in \calC^{\star}} x^{\star}_{c^{\star}j} \leq 1 \,.
\end{align*}
We next show that $\sum_{c \in \calC} x_{cj} = 1$ for all $j\in V$.
\begin{align*}
\sum_{c \in \calC} x_{cj} = \sum_{c \in \calC} \sum_{c^{\star} \in \phi^{-1}(c)}x^{\star}_{c^{\star}j} = \sum_{c^{\star}\in \calC^{\star}} x^{\star}_{c^{\star}j} = 1\,,
\end{align*}
where the second last equality follows since $\phi^{-1}$ forms a partition of $\calC^{\star}$ and 
the final equality follows from feasibility of $x^{\star}$.

We now show that $x$ satisfies~\eqref{eq:ind-fair-LP}. Fix two points $j_1$ and $j_2$ in $V$.
Recall the Definition~\ref{def:f-div} of $f$-divergence between $\vec{x_{j_1}}$ and $\vec{x_{j_2}}$:
\begin{align*}
&D_{f}(\vec{x_{j_1}} || \vec{x_{j_2}}) 
= \sum_{c\in \calC} x_{cj_2} f \left( \frac{x_{cj_1}}{x_{cj_2}}\right) 
= \sum_{c\in C} \left( \sum_{c^{\star} \in \phi^{-1}(c) }x^{\star}_{c^{\star}j_2}\right) f \left( \frac{\sum_{c^{\star} \in \phi^{-1}(c) }x^{\star}_{c^{\star}j_1}}{\sum_{c^{\star} \in \phi^{-1}(c) }x^{\star}_{c^{\star}j_2}}\right) \,.\\
\end{align*}
Observe that, for any center $c$ with $\phi^{-1}(c) = \emptyset$, $x_{cj} = 0$ for each $j\in V$.
We call such centers {\em empty} centers.
Hence, assuming $0f(\frac{0}{0})$ is well-defined, we can disregard any {\em empty} center $c$.
Fix a center $c\in C$ that is non-empty. For ease exposition, let $B = \sum_{c^{\star} \in \phi^{-1}(c) }x^{\star}_{c^{\star}j_2}$.
Since $f$ is convex, by applying Jensen's inequality, we derive the following:
\begin{align*}
    f\left( \frac{\sum_{c^{\star} \in \phi^{-1}(c) }x^{\star}_{c^{\star}j_1}}{B}\right) 
    \leq \sum_{c^{\star} \in \phi^{-1}(c) } \frac{x^{\star}_{c^{\star}j_2}}{B} f\left( \frac{x^{\star}_{c^{\star j_1}}}{x^{\star}_{c^{\star}j_2}}\right)
\end{align*}
Plugging this in the above equations, we derive:
\begin{align*}
    D_{f}(\vec{x_{j_1}} || \vec{x_{j_2}}) 
     &\leq \sum_{c\in C} \sum_{c^{\star} \in \phi^{-1}(c) } {x^{\star}_{c^{\star}j_2}} f\left( \frac{x^{\star}_{c^{\star j_1}}}{x^{\star}_{c^{\star}j_2}}\right) \,, \\
     &= \sum_{c^{\star} \in \calC^{\star}} {x^{\star}_{c^{\star}j_2}} f\left( \frac{x^{\star}_{c^{\star j_1}}}{x^{\star}_{c^{\star}j_2}}\right) \,, \\
    &= D_{f}(\vec{x^{\star}_{j_1}} || \vec{x^{\star}_{j_2}}) \\
     &\leq \fairness(x,y) \,,
\end{align*}
where (a) the first equality follows since $\phi^{-1}$ partitions the set $\calC^{\star}$, %each $c^{\star} \in \calC^{\star}$ is mapped uniquely to a $c\in \calC$,
(b) the second equality follows by definition of $D_{f}$, and (c) the last inequality follows since
$x^{\star}$ is a feasible solution. This completes the proof of the lemma.

We now prove the second part of the lemma. Fix a point $j\in V$.
Let $d^{\star}(j)$ and $d(j)$ denote
the expected cost paid by the point $j$ in the optimal solution $x^{\star}$ 
and our constructed solution $x$, respectively. Formally,
\begin{align}
    d^{\star}(j) &= \sum_{c^{\star} \in \calC^{\star}} x^{\star}_{c^{\star}j} d(j,c^{\star})^p \,, \\
    d(j) &= \sum_{c\in \calC} x_{cj} d(c,j)^p \,.
\end{align}
%Let $\cost(C)$ denote the cost of the vanilla clustering solution.
Recall that in \vkp, $j$ is assigned to its closest cluster center in $\calC$.
Assume $c_j$ is the closest center to $j$ in $\calC$: $c_j = \arg \min \{d(j,c):c\in \calC\}$.
Then, $\left(\sum_{j\in V} d(j,c_j)^p\right)^{1/p} \leq  \rho\cdot OPT_{k,p}(\calI)$. 
Further, $\vopt(\calI) \leq \iopt(\calI)$, since any solution to \kfp is also a 
feasible solution to \vkp. 
% \begin{proof}
% We begin the proof by first considering $p=1$. In this case, $\psi=1$.
% This was implicitly proved by Bera~\etal~\cite{bera19}.
% For completeness, we present a proof here as well.
% The proof follows by application of triangle inequality and definition of $\phi$.
% \begin{align}
%     d(j,c) &\leq d(j,c^{\star}) + d(c^{\star},f) \,,~~~\text{(triangle inequality)} \nonumber\\
%     & \leq d(j,c^{\star}) + d(c^{\star},c_j) \,,~~~ \text{(since $c^{\star} \in \phi^{-1}(c)$)} \nonumber\\
%     & \leq d(j,c^{\star}) + d(c^{\star},j) + d(j,c_j) \,,~~~\text{(triangle inequality)}\nonumber\\
%     &= 2 d(j,c^{\star}) + d(j,c_j) \label{eq:norm1}\,.
% \end{align} 
% To prove for general $p$ such that $p\geq 1$, we use the convexity of the function $h(x)=x^p$.
% In particular, applying Jensen inequality, we derive,
% \begin{align*}
%      &d(j,c) ^p \\ 
%      &\leq (2 d(j,c^{\star}) + d(j,c_j))^p\,, ~~\text{(from~\Cref{eq:norm1})} \\
%      & \leq 3^{p-1} (2 d(j,c^{\star})^p + d(j,c_j)^p) ~~\text{(Jensen's inequality)} \,.
% \end{align*}
% Setting $\psi=3^{p-1}$ gives the claim.
% \end{proof}

We now bound $d(j)$ in terms of $d^{\star}(j)$ and $d(j,c_j)$. Assume $\psi = 3^{p-1}$.
%As above, let $\psi = 3^{p-1}$.
\begin{align*}
d(j) &= \sum_{c\in \calC} x_{cj} \cdot d(j,c)^p \,, \\
& = \sum_{c\in \calC} \sum_{c^{\star} \in \phi^{-1}(c)} x^{\star}_{c^{\star}j} \cdot d(j,c)^p \,, \\
& \leq \psi \sum_{c\in \calC} \sum_{c^{\star} \in \phi^{-1}(c)} x^{\star}_{c^{\star}j} \cdot (2d(j,c^{\star})^p+d(j,c_j)^p) \,, ~~\text{(using~\Cref{clm:distance})} \\
&= 2\psi \sum_{c\in \calC} \sum_{c^{\star} \in \phi^{-1}(c)} x^{\star}_{c^{\star}j} \cdot d(j,c^{\star})^p 
 + \psi d(j,c_j)^p \sum_{c\in \calC} \sum_{c^{\star} \in \phi^{-1}(c)} x^{\star}_{c^{\star}j}  \,,\\
&= 2\psi \sum_{c^{\star}\in \calC^{\star}} x^{\star}_{c^{\star}j} \cdot d(j,c^{\star})^p 
+ \psi d(j,c_j)^p \sum_{c^{\star}\in \calC^{\star}}  x^{\star}_{c^{\star}j}  \,,\\
&= 2\psi d^{\star}(j) + \psi d(j,c_j)^p \,,
\end{align*}
where (1) the second last equality follows since $\phi$ is a function, and (2) the last equality
follows from the definition of $d^{\star}(j)$ and uses the fact that 
$\sum_{c^{\star}\in \calC^{\star}}  x^{\star}_{c^{\star}j}=1$.

Taking a sum over all the points in $V$, we get
\begin{align*}
    \sum_{j\in V} d(j) &\leq 2\psi \sum_{j\in V} d^{\star} (j) + \psi \sum_{j\in V} d(j,c_j)^p \\
    &\leq 3^{p-1} \left( 2\sum_{j\in V} d^{\star} (j) + \sum_{j\in V} d(j,c_j)^p \right) \\
    & \leq 3^{p-1} \left( \rho+2 \right) ^p \sum_{j\in V} d^{\star}(j)^p \\
    &= 3^{p-1}\left((\rho + 2)\iopt(\calI) \right)^p
\end{align*}
Taking the $p$-th root on both sides gives us the lemma.
\end{proof}

\subsection{Individually Fair \boldmath{$k$}-Center}
\label{appendix:k_center}
In this section, we revisit the individually fair $k$-center problem. As alluded in~\Cref{remark:k-center}, we need to be careful when 
dealing with $p=\infty$. As such, the same theorem still holds, but
the algorithmic details are slightly different. 
We first define the problem in the following way.

\begin{definition}
Assume we are given a function $f$ as in Definition~\ref{def:f-div} and a fair similarity measure $\fairness$. Then, \kcf asks for the minimum distance $R$ along with 
\begin{inparaenum}[\bfseries (1)]
\item a set of cluster centers ${\calC} \subseteq V$ of size at most $k$ and
\item  a distribution $\mu_j$ over $\calC$ for each point $j \in V$, such that any center $c\in \calC$ that lies in the support of $\mu_j$ satisfies $d(c, j) \leq R$.
\end{inparaenum}
% \begin{enumerate}
%     \item  a set of cluster centers ${\calC} = \{c_1,c_2,\ldots,c_k\}$ where each $c_k\in V$ and
%     \item a distribution $\mu_j$ over $\calC$ for each point $j \in V$,
% \end{enumerate}
Further, the following individual fairness constraints need to be satisfied by the output distributions. 
\begin{align}
\label{eq:ind-fair-kc}
D_f(\mu_{j_1} || \mu_{j_2}) \leq \fairness(j_1, j_2), \forall j_1,j_2\in V
\end{align}

\end{definition}

\mypar{Algorithmic Details} The algorithm follows exactly the same template as described for \kfp in Section~\ref{sec:alg-if}. We first use a standard 2-approximation algorithm for \vkc to determine the set $\calC$. Next we define the constrained problem \fairoptkc which is analogous to \ifa in Section~\ref{sec:alg-if}. As is standard for $k$-center problems, suppose we make the correct `guess' for the optimal radius for the \kcf problem - call it $\rs$. For any client $j$, define $B_j$ to be the ball with center at $j$ and radius $4\rs$. We define the following feasibility mathematical program. A variable $x_{cj}$ is defined  if and only if $c\in B_j$, for all $c\in \calC, j\in V$.
\begin{align}
%\label{eqn:FairLPkc}
    \fairoptkc : & \sum_{c\in \calC\cap B_j} x_{cj} = 1 ~~\forall j \in V \, \label{eqn:sum-kc}\\
    & D_{f}(\vec{x}_{j_1} || \vec{x}_{j_2}) \leq \fairness(j_1,j_2) ~~\forall j_1,j_2 \in V 
    \label{eq:ind-fair-LP-kc}\\
    & 0\leq  x_{cj} \leq 1~~\forall j\in V, c\in B_j\cap \calC \label{eqn:fraction-kc}
\end{align}

We return any feasible solution $x$ to the above constrained program as our final solution. In the remainder of the section, we prove that such a solution exists. Let $x^{\star}$ be an optimal solution to \kcf with radius $\rs$.
We again define the mapping $\phi$ from the centers in the support of $x^{\star}$ to those in $\calC$ and a potential solution $x$ to the above LP, exactly in the same way as done in Section~\ref{sec:alg-if} and subsequently used in Claim~\ref{clm:distance}. We define $\supp(x',j)$ as the set of open centers in the support of $x'_j$ for any solution $x'$ to \kcf

\begin{claim}
\label{clm:distancekCenter}
For any point $j\in V$, consider any center $c\in \supp(x,j)$. Let $c_j$ be the  closest center to $j\in \calC$. Then for each $c^{\star}\in \phi^{-1}(c)\cap \supp(x^{\star}, j)$, we have
$$ d(j,c) \leq 2d(j, c^{\star}) + d(j, c_j) $$
\end{claim}

The proof is immediate from the first part of the proof for Claim~\ref{clm:distance} and we skip that to avoid repetition. This claim will now give the following lemma.

\begin{lemma}
$x$ is a feasible solution to \fairoptkc.
\end{lemma}

\begin{proof}
The proof that $x$ satisfies the individual fairness constraints~\eqref{eq:ind-fair-LP-kc} is exactly the same as done in the proof of Lemma~\ref{lem:feasibility}. 

However, we also need to prove that $x$ satisfies the constraints~\eqref{eqn:sum-kc}. Consider any point $j\in V$. Let $c_j$ be the closest center to $j$ in $\calC$. Recall that $x^{\star}$ is an optimal solution to \kcf. %Let $\supp(x^{\star},j)$ be the subset of centers in the support of the distribution $x^{\star}$. 
Clearly $d(c^{\star},j) \leq \rs$ for any $c^{\star}\in \supp(x^{\star},j)$. Also, by definition of the mapping $\phi$, $\sum_{c\in \supp(x,j)} x_{cj} = \sum_{c^{\star}\in \supp(x,j)\cap \phi{-1}(c)} x^{\star}_{c^{\star},j} = 1$, by feasibility of $x^{\star}$. Now consider any $c\in \supp(x, j)$. By Claim~\ref{clm:distancekCenter}, $d(c,j) \leq 2d(j, c^{\star}) + d(j, c_j) $ for any $c^{\star}\in \supp(x^{\star},j)\cap \phi^{-1}(c)$. We use the following three facts --- (1) $\calC$ is a set of centers for a $2$-approximate solution to \vkc, (2) an optimal solution to \kcf is a feasible solution to \vkc, and (3) $x^{\star}$ is an optimal solution to \kcf with radius $\rs$. This gives us $d(c, j) \leq 4\rs$ and we are done.
\end{proof}

Combining all of the above, we have the following theorem.
\begin{theorem}
	\label{thm:individual_fairness_kc}
	There exists a 4-approximation algorithm for \kcf.
\end{theorem}

\mypar{Hardness of Individually Fair \boldmath{$k$}-Center}
The NP-hardness of \kcf follows almost immediately from the hardness of \vkc. Suppose $D_{\TV}$ is the choice for $f$-divergence and the fairness similarity measure is $\fairness=d$, the underlying metric. It is a standard fact that the hard instances of \vkc arise from a metric defined by $d(v, v') = 1$ or $d(v, v') = \infty$ (here $\infty$ is a very large number). Now suppose $\calI$ is such an instance of \vkc. Then we have the following lemma.

\begin{lemma}
The instance $\calI$ has a solution with $k$ centers and radius 1 if and only if the corresponding \kcf instance has a solution with radius 1.
\end{lemma}

\begin{proof}
Suppose $\calI$ is a `yes' instances to \vkc with radius $1$ and suppose $\psi(j)$ be the center to which $j$ has been assigned in such a solution. Now consider the solution to \kcf where for $j$, we return the distribution with $\psi(j)$ as the only center in its support. Since $D_{\TV}$ can take a value of at most $1$ and all distances are either 1 or $\infty$, this solution is individually fair. 

Conversely, if there exists a solution to \kcf with radius 1, then trivially, there exists a solution to \vkc with radius 1. 
\end{proof}
% In proving~\Cref{thm:individual_fairness}, we assumed
% the following claim. We present its proof here.

\subsection{On the Price of Achieving Individual Fairness}
\label{appendix:cost}
In this section, we discuss the {\em price} associated with achieving individual fairness. More specifically, we call the ratio of the optimal cost of an individually fair clustering instance, to that of the clustering
instance without the fairness constraints, as the {\em price} of achieving fairness. We give a simple example to show that, perhaps unsurprisingly, the {\em price} of achieving fairness can be arbitrarily large depending on the underlying fairness measure $\fairness$. 

Recall that, for an instance $\cI$ of {\vkp} and {\kfp}, we denote the corresponding optimal costs as 
$\vopt(\cI)$ and $\iopt(\cI)$, respectively. 
We show that the ratio of $\iopt(\cI)$ to that of $\vopt(\cI)$
can be arbitrarily large, depending on the fairness measure $\fairness$.
In~\Cref{fig:cost_of_fairness}, the input instance $\cI$ consists of data points on a line and assume $d(u_2,v_1)\geq R>>r \geq d(u_{1},u_2)\approx d(v_1,v_2)$. Further, assume $k=2$ and $p=1$ ($k$-median). Then, $\vopt(\cI) = O(r)$.
Now let $\fairness(u,v)=\eps$ for some small positive constant $\eps$,
and the measure of the individual fairness is the {\em total variation norm} $D_{\TV}$. Then, in any solution to the \kfp instance, $D_{\TV}(\mu_u || \mu_v) \leq \eps$. This implies, $\iopt(\cI) = O(R) >> \vopt(\cI)$.
A similar argument is true for the case of $k$-means ($p=2$) and $k$-center ($p=\infty$) as well.
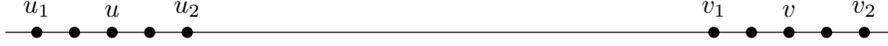
\begin{figure}[!ht]
    \centering
\begin{tikzpicture}[dot/.style={circle,inner sep=1.5pt,fill,label=#1},
  extended line/.style={shorten >=-#1,shorten <=-#1},
  extended line/.default=1cm]
\node [dot=$u_1$,name=A] at (1,0) {};
\node [dot,name=B] at (1.5,0) {};
\node [dot=$u$,name=C] at (2,0) {};
\node [dot,name=D] at (2.5,0) {};
\node [dot=$u_2$,name=E] at (3,0) {};

\node [dot=$v_1$,name=V] at (10,0) {};
\node [dot,name=W] at (10.5,0) {};
\node [dot=$v$,name=X] at (11,0) {};
\node [dot,name=Y] at (11.5,0) {};
\node [dot=$v_2$,name=Z] at (12,0) {};

\draw [extended line=0.5cm] (A) -- (Z);

\end{tikzpicture}
    \caption{The {\em price} of achieving individual fairness. Let $k=2$ and $p=1$. Assume $d(u_1,u_2) \approx d(v_1,v_2) \leq r$, $d(u_2,v_1) \geq R$, and $R >> r$. Then, $\vopt = O(r)$. Now assume $\fairness(u,v) = \eps$. Then, $\iopt = O(R) >> \vopt$. }
    \label{fig:cost_of_fairness}
\end{figure}

In this toy example, we have shown that the choice $\fairness$ plays
an important role in determining the {\em price} of achieving fairness. 
In our experiments (\Cref{sec:exp}), we demonstrate a similar effect in  real-world scenarios. We consider two different fairness similarity measures. The first one, $\fairness_1$, is simply the underlying metric feature space $d$. The second one, $\fairness_2$, is an
asymmetric notion where $\fairness(i,j)$ is decided
based on the ``small'' local neighborhood information of $i$ in the feature space (see~\Cref{sec:exp} for exact details). In~\Cref{fig:if_alg_opt}, we 
compare the cost of $\iopt$ vs $\hkm$ (which approximates $\vopt$).
We observe that for $\fairness_1$, the {\em price} of achieving fairness is quite large. In comparison, for $\fairness_2$, we can achieve fairness by almost paying the same cost as that of vanilla solutions.

We emphasize that this discussion is not be confused with the theoretical guarantees of our algorithm (\Cref{alg:kfp}) --- there we bound the cost of our solution with respect to $\iopt(\cI)$. The current discussion, on the other hand, studies the value of $\iopt(\cI)$ itself and highlights the impact of the fairness measure in determining the {\em price} of achieving fairness.

\section{Missing Details from~\Cref{sec:IF_SF}}
\label{appendix:IF_SF}
In this section, we revisit the \total problem and spell out the missing
details from~\Cref{sec:IF_SF}. For the sake of completeness, we first
restate the problem definition.
\begin{definition}[\total]
\label{def:appendx_stat_indv_fairness}
Assume we are give an instance of the \kfp problem. Additionally, we are given $\ell$-many 
(possibly overlapping) protected groups $G_1,G_2,\ldots,G_{\ell}$
and for each such group we are given two input {\em group} fairness parameters $\alpha_i$ and $\beta_i$.
The goal is to output
\begin{inparaenum}[\bfseries (1)]
\item a set of cluster centers ${\calC} \subseteq V$ of size at most $k$ and
\item  a distribution $\mu_j$ over $\calC$ for each point $j \in V$,
\end{inparaenum}
such that 
\begin{enumerate}
    \item For each cluster, the expected fraction of the points from group $G_i$ lies between $\beta_i$ and $\alpha_i$,
    \item $D_f(\mu_{j_1} || \mu_{j_2}) \leq \fairness(j_1,j_2)$ for each pair of points $j_1,j_2 \in G_p$, for all $p\in [\ell]$.
\end{enumerate}
 The objective is to minimize 
 $   \calL_p(\mu, \calC) := \left( \sum_{j\in V} \Ex_{c{\sim}\mu_j}(d(j,c)^p) \right)^\frac{1}{p}$.
\end{definition}
\mypar{Why is Individual Fairness Enforced only Inside Protected Groups?}

\begin{figure}[!ht]
\centering
\begin{minipage}[b]{.6\textwidth}
\centerline{\includegraphics[scale=0.4,bb= {0 0 20in 20in},trim=3cm 2cm 3cm 2cm, clip=true]{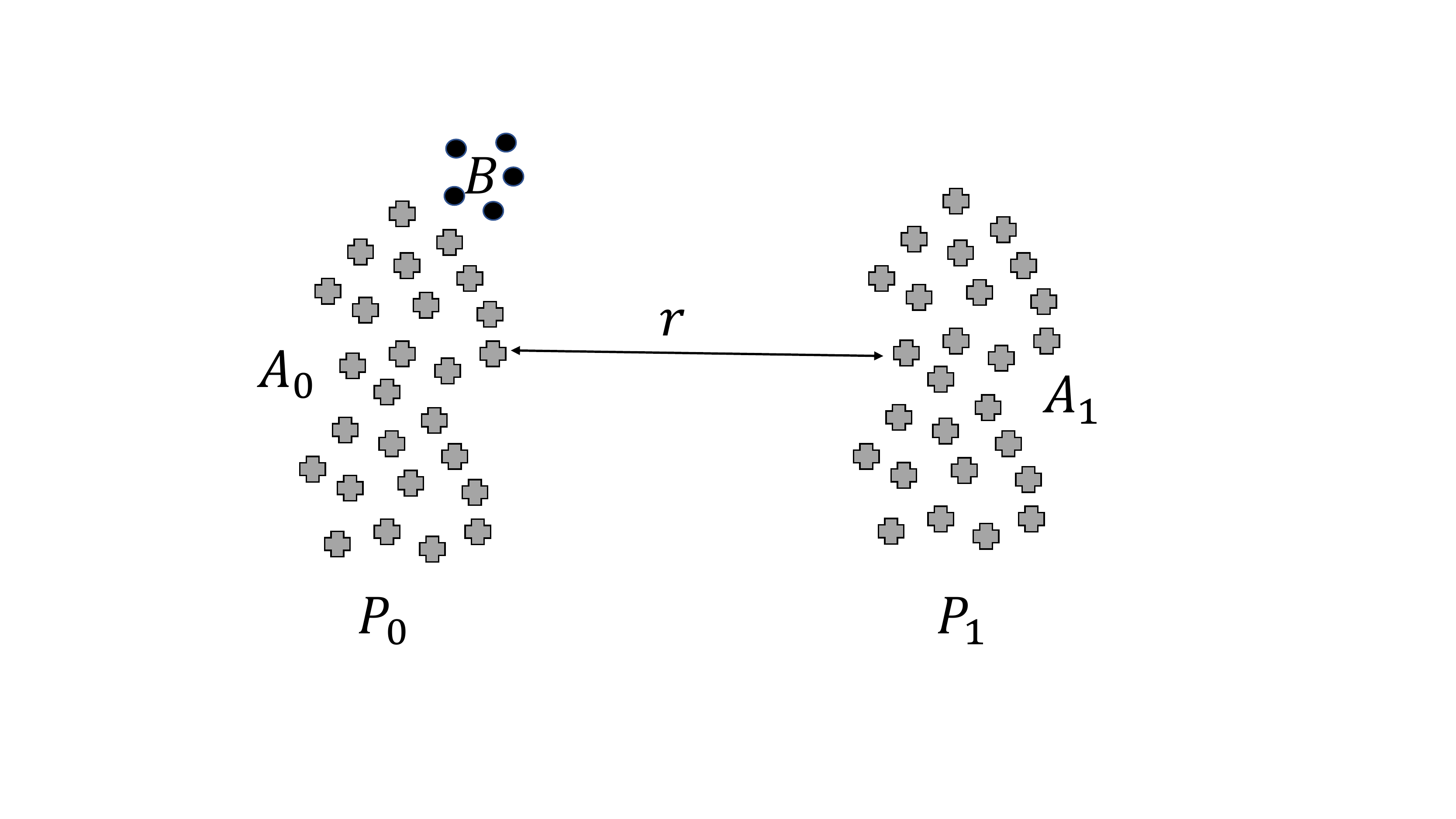}}
\caption{Combining Individual and Group Fairness}
\label{fig:gfInside}
\end{minipage}
\end{figure}
 In Fig~\ref{fig:gfInside}, suppose the entire population of size $N$ is partitioned into two protected groups $A$ and $B$ according to some protected attribute. Let $P_0$ and $P_1$ be two sets of closely packed points separated by a distance $r$, which is a very large number. Further, let $A_0 = P_0 \cap A, A_1 = P_1 \cap A = P_1$ and $B\subset P_0$. Suppose $|B| = \frac{N}{10}$ and $|A_0| = |A_1|$. Hence, the statistical fairness constraints require $10\%$ of each cluster to be formed with points from $B$ and $90\%$ from $A$. If we impose individual fairness only inside $B$ and $A$, then a reasonable solution would be to assign all points in $A_0$ to centroid of $P_0$ and $A_1$ to that of $P_1$, each with probability 1. Further, we can assign each point in $B$ to each of the centers with probability 0.5.
 
 On the other hand, imposing individual fairness across every pair of points requires that points in $B$ and $A_0$ have roughly the same distribution, since the radius of $P_0$ is very small compared to $r$. As before, due to statistical fairness conditions, the distributions of $A_0$ and $A_1$ also needs to be approximately the same. %\yl{why is this the case? is it because of $r$ is large? or because the radius in G0 is small?} 
 This will result in a trivial solution where each point is assigned to each centroid with roughly probability 0.5. %\yl{hmm unclear why having the same probability leads to 0.5 each; previously we assign points in $S$ with 0.5 probability is because T0 and T1 are assigned to different clusters separately. but now T0 is also assigned randomly the same way as S, so I wonder how this nuounce will change the 0.5 argument} 
 The discussion above closely follows a similar discussion in the paper by Dwork et al.~\cite{dwork2012fairness} where they term the notion of combined fairness as {\em fair affirmative action}. 

\mypar{On the Feasibility of a Combined Fair Clustering Instance}Before describing an algorithm for computing an {\em approximate} solution, we address the question of finding a feasible solution to the \total problem instance. Note that in the absence of the group fairness constraints, it is trivial to construct a feasible solution to the
\kfp problem. Indeed, we can simply assign to each point $j$ in $V$
a uniform distribution $\mu_j$ over any arbitrary set of clusters centers $\cC$ ($|{\cC}| \leq k$). By definition of $f$-divergence, $D_f(\mu_{j_1}||\mu_{j_2}) = 0$ for all pair of points $j_1,j_2 \in V$.
Since, the fair similarity measure $\fairness$ is non-negative, this satisfies the individual fairness constraints (~\Cref{eq:ind-fair}).
Can we verify the feasibility of a \total instance {\em efficiently}?
We answer this question in affirmative. In fact, we give a simple
condition in the following claim for the existence of a feasible solution. We remark that such a claim holds true for the
group fairness problem considered in~\cite{bera19,Bercea2018} as well.
\begin{claim}
\label{claim:combine_feasibility}
Given an instance $\cI=(V,k,f,\fairness,\alpha,\beta)$ to 
\total, there exists a feasible solution to $\cI$, iff the following 
condition is true:
\begin{align}
\label{eqn:condition}
    \beta_r |V| \leq |G_r| \leq \alpha_r |V|~~\forall r \in [\ell] 
\end{align}
\end{claim}
\begin{proof}
We first prove the ``if'' direction. 
For each point $j\in V$, let $\mu_j$ be a uniform distribution 
over an arbitrary set of $k$ centers $\cC\subseteq V$.
Then, $\mu_j$ is a feasible solution to $\cI$. We have already
argued above that such uniform distributions trivially satisfy
individual fairness constraints between each pair of points in 
$V$, and hence for each pair of points inside each protected group as well. Now, fix a cluster $c\in \cC$ and a protected group 
$G_r$, for $r\in [\ell]$.
The expected number of points assigned to the cluster
$c$ from the group $G_r$ is $|G_r|/|\cC|$.
The expected size of the cluster with cluster center $c$ is $|V|/|\cC|$.
Then, the condition in~\Cref{eqn:condition} immediately implies group fairness.

We now prove the ``else if'' direction. Let $(\cC,\{\mu_j\}_{j\in V})$ be some feasible solution to the instance. Let $\mu_j[c]$ denote the probability of assigning $j$ to the cluster center $c$. We then use the sub-additive property of the group fairness constraints to argue that ~\Cref{eqn:condition} must hold. 
More formally, group fairness implies for each center $c\in \cC$ and for each $r \in [\ell]$, we have
\begin{align*}
    \beta_r \sum_{j\in V} \mu_{j}[c] \leq \sum_{j\in G_r} \mu_{j}[c] \leq \alpha_r \sum_{j\in V} \mu_{j}[c] \,.
\end{align*}
Summing over all $c\in \cC$ and rearranging, we get
\begin{align*}
    \beta_r \sum_{c\in \cC}\sum_{j\in V} \mu_{j}[c] &\leq \sum_{c\in \cC}\sum_{j\in G_r} \mu_{j}[c] \leq \alpha_r \sum_{c\in \cC} \sum_{j\in V} \mu_{j}[c] \,, \\
    \Rightarrow \beta_r \sum_{j\in V} \sum_{c\in \cC} \mu_{j}[c] &\leq \sum_{j\in G_r}\sum_{c\in \cC} \mu_{j}[c] \leq \alpha_r \sum_{j\in V} \sum_{c\in \cC}  \mu_{j}[c] \,, \\
    \Rightarrow \beta_r |V| &\leq |G_r| \leq \alpha_r |V| \,.
\end{align*}
This completes the proof of the claim.
\end{proof}

\mypar{Algorithm for the Combined Fair Clustering Problem}We now discuss our algorithm for solving the \total problem.
Recall that, our algorithmic strategy is to first solve the \vkp problem on the input instance to find the cluster centers and then use these cluster centers to solve a
fair assignment problem. For completeness we present it 
formally in~\Cref{alg:total}. We describe the fair assignment problem
as an optimization problem below and denote it as the \totalassgn problem.
\begin{align}
\label{eqn:total_fair_assgn}
    \min& \sum_{j\in V}\sum_{c\in \calC} x_{cj} d(i,j)^p  \\
    \text{s.t.}~ & \sum_{c\in \calC} x_{cj} = 1 ~~\forall j \in V \, \nonumber\\
    & D_{f}(\vec{x_{j_1}} || \vec{x_{j_2}}) \leq 
\fairness(j_1,j_2) ~~\forall r \in [\ell] , j_1,j_2 \in G_r 
     \nonumber\\
    & \beta_r \sum_{j\in V} x_{cj} \leq  \sum_{j\in G_r} x_{cj} \leq \alpha_r \sum_{j\in V} x_{cj}, \forall c\in \calC\,, r \in [\ell]  \nonumber\\ %\label{eq:group-fair-LP2}
    & 0\leq  x_{ij} \leq 1 \nonumber
\end{align}
The second constraint enforces individual fairness between points in the same protected group and the third constraint ensures group fairness on the solution. 

We now prove~\Cref{thm:combined_fairness} that captures our main result on the \total problem. For completeness, we restate the theorem here.
\IFSF*
\begin{proof}
The proof of this theorem follows along the line of the proof of 
~\Cref{thm:individual_fairness}. 
\begin{algorithm}[!ht]
    \caption{Algorithm for \total}
    \label{alg:total}
    \begin{algorithmic}[1]
        \STATE {\bfseries $\algtf(\calI)$} 
        \STATE Use a $\rho$-approximation algorithm for \vkp on $\calI$ --- let $\calC$ be the set of centers.
        \STATE Solve the \totalassgn~ problem on instance $\calJ = (V,\calC,f,\alpha,\beta)$ --- let ${\mu}$ be the solution. 
        \STATE  return $(\calC,{\mu})$
    \end{algorithmic}
\end{algorithm}

Assume $(\calC^{\star},x^{\star})$ be an optimal solution to the instance $\calI$ of \total
and $\algtf(\calI)$ returns $(C,\mu)$.
We construct a feasible solution $x$ to the \totalassgn($\calJ=(V,C,f,\fairness,\alpha,\beta)$) using $\calC^{\star}$ and $x^{\star}$.
$\algtf$ outputs the optimal solution to \totalassgn, hence
$\calL_p(\mu, \calC) \leq \calL_p(x, \calC)$. So, as in the proof of~\Cref{thm:individual_fairness},
it is sufficient to bound $\calL_p(x, \calC)$ to prove the approximation ratio of $\algtf$.

Recall the definition of the {\em nearest} function $\phi$ and its inverse:
$\phi(c^{\star}) = \arg\min_{c\in \calC} d(c,c^{\star})$ for each $c^{\star} \in \calC^{\star}$,
and $\phi^{-1}(c) = \{c^{\star}\in \calC^{\star}: \phi(c^{\star})=c\}$ for each $c \in \calC$.
For each $j\in V$ and each $c\in \calC$, 
set $x_{cj} = \sum_{c^{\star} \in \phi^{-1}(c) }x^{\star}_{c^{\star}j}$. In words,
for a fixed point $j\in V$ and a fixed center $c\in \calC$, we look at the centers 
in the optimal solution that are mapped to $c$ by $\phi$, and sum the 
corresponding probabilities to get $x_{cj}$. In the remaining, we 
prove that $x$ is a feasible solution to \totalassgn$(\calJ)$ and bound its cost.
\begin{lemma}
\label{lem:total:feasibility}
$x$ is feasible to \totalassgn($\calJ$).
\end{lemma}
\begin{proof}
It follows from the proof of first part of~\Cref{lem:feasibility} that
$x$ satisfies all the constraints in the \totalassgn LP (~\cref{eqn:total_fair_assgn})
barring the {\em group fairness} constraints. The {\em group fairness} follows
by the sub-additivity of the constraints. We show this formally below. For any center $c\in \calC$,
if $\phi^{-1}(c)=\emptyset$, then the corresponding {\em group fairness} constraints are
trivially satisfied. Now assume $\phi^{-1}(c)\neq \emptyset$. Fix a group $G_r$. 
\begin{align*}
    \sum_{j \in G_r} x_{cj} &= \sum_{j \in G_r}  \sum_{c^{\star} \in \phi^{-1}(c)} x^{\star}_{c^{\star}j} \,, \\
    &= \sum_{c^{\star} \in \phi^{-1}(c)} \sum_{j \in G_r}   x^{\star}_{c^{\star}j} \,,\\
    &\leq \sum_{c^{\star} \in \phi^{-1}(c)} \alpha_r \sum_{j\in V} x^{\star}_{c^{\star}j}\,,~~\text{(by optimality of $x^{\star}$)} \\
    &= \alpha_r  \sum_{j\in V} \sum_{c^{\star} \in \phi^{-1}(c)} x^{\star}_{c^{\star}j} \,, \\
    &= \alpha_r \sum_{j\in V} x_{cj} \,.
\end{align*}
Similarly, we can show that $\sum_{j\in G_r} x_{cj} \geq \beta_r  \sum_{j\in V} x_{cj}$, proving that
$x$ is a feasible solution to the \totalassgn($\calJ$) LP.
\end{proof}

We next bound the cost of the solution $x$.
\begin{lemma}
\label{lem:total:cost}
$\calL_p(x, \calC) \leq 3^{\left( 1-\frac{1}{p} \right) } \left( \rho+2 \right)\cdot \calL_p(x^{\star}, \calC^{\star})$.
\end{lemma}
The proof of this lemma is identical to the proof of second part of~\Cref{lem:feasibility}. Together~\Cref{lem:total:feasibility,lem:total:cost} prove the claim in the
theorem.
\end{proof}

\mypar{Combined Fair \boldmath{$k$}-Center Problem}
We remark here that the Combine Fair $k$-center problem
needs to be treated slightly differently, as we discussed in
~\Cref{appendix:k_center}. The details are analogous, and we refrain
from repeating them here. 

\section{Individual Fairness to Group Fairness when $\fairness = d$}
\label{appendix:relation}

In this section, we explore the connection between the notion of 
individual fairness and group fairness for the special case of
$\fairness=d$. In particular, we are interested in 
finding conditions under which individually fair clustering solutions guarantees group fairness. Such connections are well-known in the context
of classification problems~\cite{dwork2012fairness}. We show that similar connections exist in the clustering context. 

Before we discuss the technical details, it is perhaps imperative to discuss the
apparent tension between the two notions of fairness in the context of clustering. 
Individual fairness is modeled after the concept of ``equality of treatment'' whereas group fairness is modeled after ``equality of outcome''. In clustering, to ensure the later,
some point $v$ might be assigned to a cluster center that is not the closet center to $v$. 
However, such assignments might be unfair to $v$ if its close neighbors are assigned to a 
center that was in fact closest to $v$ as well. Indeed, we demonstrate this aspect in~\Cref{fig:combined_fairness}.
Nevertheless, if the ``spread'' of the points from each protected group are ``similar''
to each other in the metric space, we expect that individual fairness will lead to
equality of outcome as well. In this section, we formally establish this connection.

We define a quantity {\em maximum additive violation}, and denote it as $\MAD$,
which captures the unfairness of the clusters for each protected group. This quantity helps us establish a connection between group fairness and individual fairness. Recall that,
a solution to the \kfp instance assigns to each point $x\in V$ a distribution $\mu_x$ over the set of
cluster centers $\calC$. For
each protected groups $G_1,G_2,\ldots,G_{\ell}$, let $p_r= |G_r|/|V|$ for each $r\in [\ell]$.
Then, for the protected group $G_r$, we define $\MAD_r$ as follows:
\begin{align}
    \MAD_r = \max_{i \in \calC} \left|  \sum_{j\in G_r} \mu_{x}(i) - p_r\sum_{j\in V} \mu_{x}(i)\right|
\end{align}
To elucidate further on this definition, note that in the entire population,
$p_r$ fraction of the points belong to group $G_r$, and hence, we expect that
each cluster will have the same proportional representation for group $G_r$. 
The notion of $\MAD_r$ captures the additive deviation from this expected number.
We remark here that our notion of $\MAD_r$ is consistent with the 
most general group fairness constraints defined in prior works~\cite{bera19,Bercea2018}. 
Indeed, the $(\alpha,\beta)$-group fairness formulation~\cite{bera19} aims at 
providing a desired bound on the quantity $\MAD_r$.
For each protected group $G_r$, we define a distribution
$\nu_{G_r}$ as a uniform distribution over the set of points in $G_r$.
In particular, $\nu_{G_r}(x) = 1/|G_r|$ if $x\in G_r$, and $0$ otherwise.
Let $\nu_V$ be the uniform distribution over the set of all points:
$\nu_V(x) = 1/|V|$ for all $x\in V$. Let $d_{\EM}(S,T)$ be the 
Earthmover's distance between the distribution $S$ and $T$, introduced formally in~\cite{RubnerTG98}.
Our main result of this section is the following lemma.
% proof of which we defer to~\Cref{appendix:relation}.
\begin{restatable}{lemma}{relation}
\label{lem:relation}
Let $(\calC, {\mu})$ be any feasible solution to the \kfp problem 
instance with $D_f$ as the statistical similarity measure. Further suppose the $f$-divergence $D_f(\mu_x||\mu_y) \geq D_{\TV}(\mu_x, \mu_y)$ for all pair of points $x,y\in V$. 
Then, For each group $G_r$, $\MAD_r \leq |G_r| \cdot d_{\EM}(\nu_{G_r},\nu_V)$
\end{restatable}
 \begin{proof}
For a fixed group $G_r$, we first define the notion of bias of the group, which is essentially an upper bound on the quantity $\MAD_r$. For convenience, let us extend the notation of $\MAD_r$ to $\MAD_r(\mu,D_f)$ to include the underlying $f$-divergence function and the distributions $\mu_x$.   

Let $\mu'_x$ be a bi-point distribution defined over the set of all points $x\in V$ that satisfy the individual fairness constraints. Then, 
\begin{align*}
\bias(r,D_f) = \max_{\mu'} \MAD_r(\mu',D_f)     
\end{align*}
Note that we have restricted the definition of  $\bias(r,D_f)$ with respect to all possible bi-point distributions. In the next claim, we justify this. We remark here that a similar observation is made by
~\cite{dwork2012fairness} in the context of classification problems.
\begin{claim}
Suppose we have distributions $\mu_x$ defined over a set of $k$ points for all $x\in V$, where $k \geq 2$, where the distributions $\mu$ satisfy the individual fairness constraints~\eqref{eq:ind-fair}. Then,
\begin{align*}
\max_{\mu} \MAD_r(\mu ,D_f) \leq \bias(r,D_f)     
\end{align*}
\end{claim}
\begin{proof}
Suppose we are given the distributions $\mu_x, x\in V$ over a set of centers $\calC$. We define  corresponding bi-point distributions $\mu'_x$ over a set $\calC'=\{c_1, c_2\}$ on two centers. Let $\calC_1 = \{i\in \calC : \mu_{G_r}(i) > p_r\mu_{V}(i)\}$ and $\calC_2 = \calC\setminus \calC_1\}$. Assign $\mu_x'(c_1) = \mu_x(\calC_1), \mu_x'(c_2) = \mu_x(\calC_2)$. First we claim that $\mu'_x, \forall x\in V$ satisfies the individual fairness constraints~\eqref{eq:ind-fair}. The proof is exactly the same as that in Lemma~\ref{lem:feasibility}. 

Next we prove that $\max_{\mu} \MAD_r(\mu ,D_f) \leq \bias(r,D_f) $. This follows using the definitions.
\begin{align*}
&\MAD_r(\mu ,D_f) \\
&= \max_{i\in \calC} |\mu_{G_r}(i) - p_r\cdot \mu_{V}(i) | \\
&=  \max \left\{\max_{i'\in \calC_1} (\mu_{G_r}(i') - p_r\cdot \mu_{V}(i')), 
\max_{i'\in \calC_2} (\mu_{G_r}(i') - p_r\cdot \mu_{V}(i'))\right\} \\
&\leq \max \left\{\sum_{i'\in \calC_1} (\mu_{G_r}(i') - p_r\cdot \mu_{V}(i')),
\sum_{i'\in \calC_2} (p_r\cdot \mu_{V}(i') - \mu_{G_r}(i'))\right\} \\
&= \max_{i\in \{c_1,c_2\}} |\mu'_{G_r}(i) - p_r\cdot \mu'_{V}(i) | \\
&\leq \bias(r, D_f)
\end{align*}
\end{proof}

We first show that if the $f$-divergence function is indeed $D_{\TV}$, then the above lemma holds. The proof follows a framework similar to that in~\cite{dwork2012fairness}. However, we need to make non-trivial modifications to handle our definition of $\MAD_r$.

We show how to upper bound the quantity $\mu'_{G_r}(c_1) - p_r\cdot \mu'_V(c_1)$. An analogous proof can be done for $c_2$. The high-level idea of the proof is as follows. We write a maximization linear program that finds the bi-point distributions $\mu_x'$ for all $x\in V$ that satisfy individual fairness with respect to $D_{\TV}$. The dual to a relaxation of this program will turn out to be the minimization linear program, whose solution gives exactly the the Earthmover's distance between $\nu_{G_r}, \nu_{V}$, up to a scaling factor of $\gr$. The claim then follows from weak duality. 
\begin{align*}
    \LPB: \max &\sum_{x\in V}\nu_{G_r}(x)\cdot \mu'_x(c_1) \\ &- \gr\cdot \sum_{x\in V} \nu_V(x)\cdot \mu'_x(c_1) \\
    \text{s.t.: } \mu'_x(c_1) + \mu'_x(c_2) &= 1 \\
    \mu'_x(c_1) - \mu'_y(c_1) &\leq d(x,y), \forall x,y\in V \\
    \mu'_x(c_1) &\geq 0, \forall x\in V
\end{align*}

Here $\mu'_x(c_1), \mu'_x(c_2), \forall x\in V$ are the variables. The first constraint ensures that they form a distribution while the second one enforces the individual fairness constraints with respect to $D_{\TV}$. Here we are using the fact that $D_{\TV} (\mu'_x, \mu'_y) \leq d(x,y)$ is equivalent to $|\mu'_x(c_1) - \mu'_y(c_1)| \leq d(x,y)$ since the distribution is bi-point. Note that we have to write this constraint for every ordered pair $x,y\in V$.  We relax the above LP by removing the first set of constraints and take the dual. 
\begin{align*}
    \LPBD: \min &\sum_{x,y\in V} \lambda(x,y)d(x,y) \\
    \text{s.t.: } \sum_{y\in V} \lambda(x,y) &\geq \sum_{y\in V} \lambda(y,x) \\
    & + \nu_{G_r}(x) - \gr\nu_{V}(x), \forall x\in V \\
    \lambda(x,y) &\geq 0, \forall x,y\in V
\end{align*}

Finally, recall that the Earthmover's distance between the distributions $\nu_{G_r}$ and $\nu_{V}$ is given by the following LP.
\begin{align*}
    \LPEM: \min &\sum_{x,y\in V} \lambda(x,y)d(x,y) \\
    \text{s.t.: } \sum_{y\in V} \lambda(x,y) &= \nu_{G_r}(x), \forall x\in V \\ 
    \sum_{y\in V} \lambda(y,x) &= \nu_{V}(x) \forall x\in V \\
    \lambda(x,y) &\geq 0, \forall x,y\in V
\end{align*}

Now, for any feasible solution $\lambda^{\star}$ to \LPEM, we can create a feasible solution to \LPBD as follows. For $\lambda^{\star}(x,y)$ appearing in the first set of constraints, we define the corresponding $\hat{\lambda}(x,y)$ for \LPBD to be the same. For $\lambda^{\star}(x,y)$ appearing in the second set of constraints, we set  $\hat{\lambda}(y,x) = \gr\lambda^{\star}(y,x)$. It is straightforward to observe that $\hat{\lambda}$ is a feasible solution to \LPBD. Putting everything together and using weak duality, we can conclude that the optimal solution to \LPB is upper bounded by $\gr$ times the Earthmover's LP optimal, and we are done.

Finally, if $D_f(\mu_x || \mu_y) \geq D_{\TV}(\mu_x, \mu_y)$, then any set of distributions which satisfies individual fairness with respect to $D_f$ will also form a feasible solution to \LPB. Hence, we have the lemma.
\end{proof}
\section{Additional Experiments}
\label{supp:exp}

In this section, we present additional experiments and the plots mentioned in~\Cref{sec:exp}, for all the datasets. We also show the practical running time of $\algif$. 

% Running time
% Table 1: Running time of ALG-IF for creditcard dataset
\begin{table}[ht]
\centering
\caption{Running time of $\algif$ for $k = 4$, $m = 250$, enforcing $\fairness_1$, on creditcard dataset for different sample sizes}
\vspace{2mm}
\begin{tabular}{ c c c c c c }
\hline
 \Tstrut
 \textbf{Number of sampled points} & 500 & 1000 & 2000 & 3000 & 4000 \\
 \textbf{Time (in seconds)} & 80 & 436 & 2901 & 10113 & 32896 \Bstrut\\
 \hline
\end{tabular}
\label{app:tab:runtime}
\end{table}

{\bf Running time.}
In this paper, we provide a generic framework and do not emphasize on running time optimization. \Cref{app:tab:runtime} shows the running time of $\algif$ on creditcard dataset for $k = 4$ and $m = 250$, enforcing fairness similarity $\fairness_1$. Although we solve a linear program with around $10,000,000$ constraints and variables, we observe that CPLEX solves it in around 9 hours. 

% SKM violations
\begin{table}[ht]
\centering
\caption{Percentage of individual fairness constraint violations of $\skm$ when $\skm$ and $\algif$ incur the same clustering cost.}
\vspace{1mm}
\subfloat[][Fairness similarity $\calF_1$]{
     \begin{tabular}{ c c c c c }
     \label{app:tab:skmviolations_global}
         Clusters ($k$) & $4$ & $6$ & $8$ & $10$ \\
         \hline
          Bank   & 95 & 98 & 99 & 99 \\
          Adult   & 88 & 94 & 98 & 99 \\
          Creditcard   & 61 & 76 & 83 & 85 \\
          Census1990   & 25 & 34 & 44 & 50 \\
          Diabetes   & 53 & 68 & 63 & 82 \\
         \hline
    \end{tabular}
}
\qquad
\subfloat[][Fairness similarity $\calF_2$]{
    \begin{tabular}{ c c c c c }
    \label{app:tab:skmviolations_local}
        Clusters ($k$) & $4$ & $6$ & $8$ & $10$ \\
        \hline
        Bank   & 4 & 4 & 4 & 4 \\
        Adult   & 4 & 5 & 7 & 8 \\
        Creditcard   & 6 & 5 & 6 & 6 \\
        Census1990   & 7 & 11 & 13 & 11 \\
        Diabetes   & 4 & 4 & 7 & 7 \\
        \hline
    \end{tabular}
}
\label{app:tab:skmviolations}
\end{table}

% SKM violations plot
\begin{figure}[ht]
    \centering
    \subfloat[][Fairness similarity $\calF_1$] {
        \includegraphics[width=0.45\linewidth]{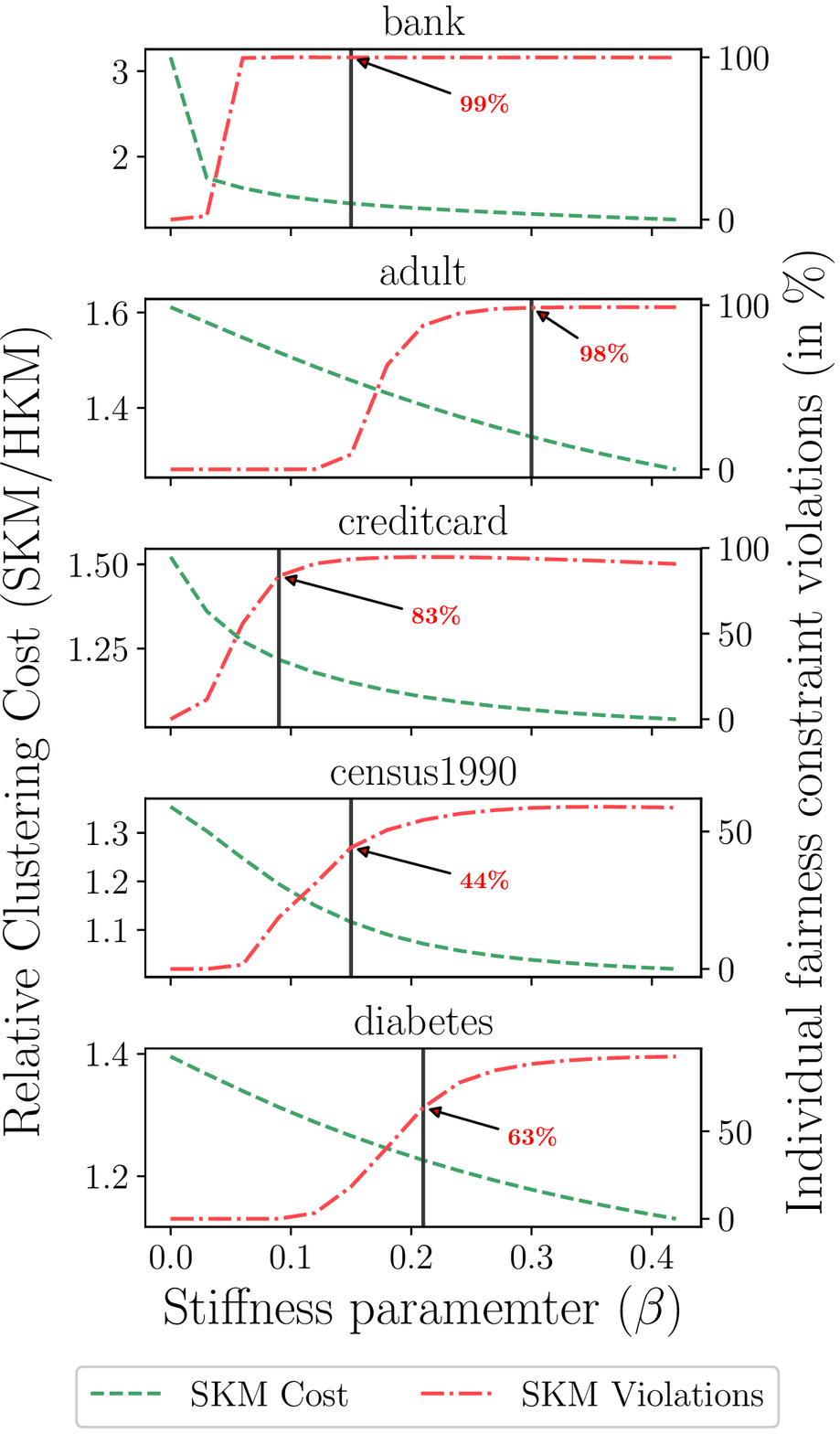}
        \label{app:fig:skmviolations_global}
    }
    \qquad
    \subfloat[][Fairness similarity $\calF_2$] {
        \includegraphics[width=0.45\linewidth]{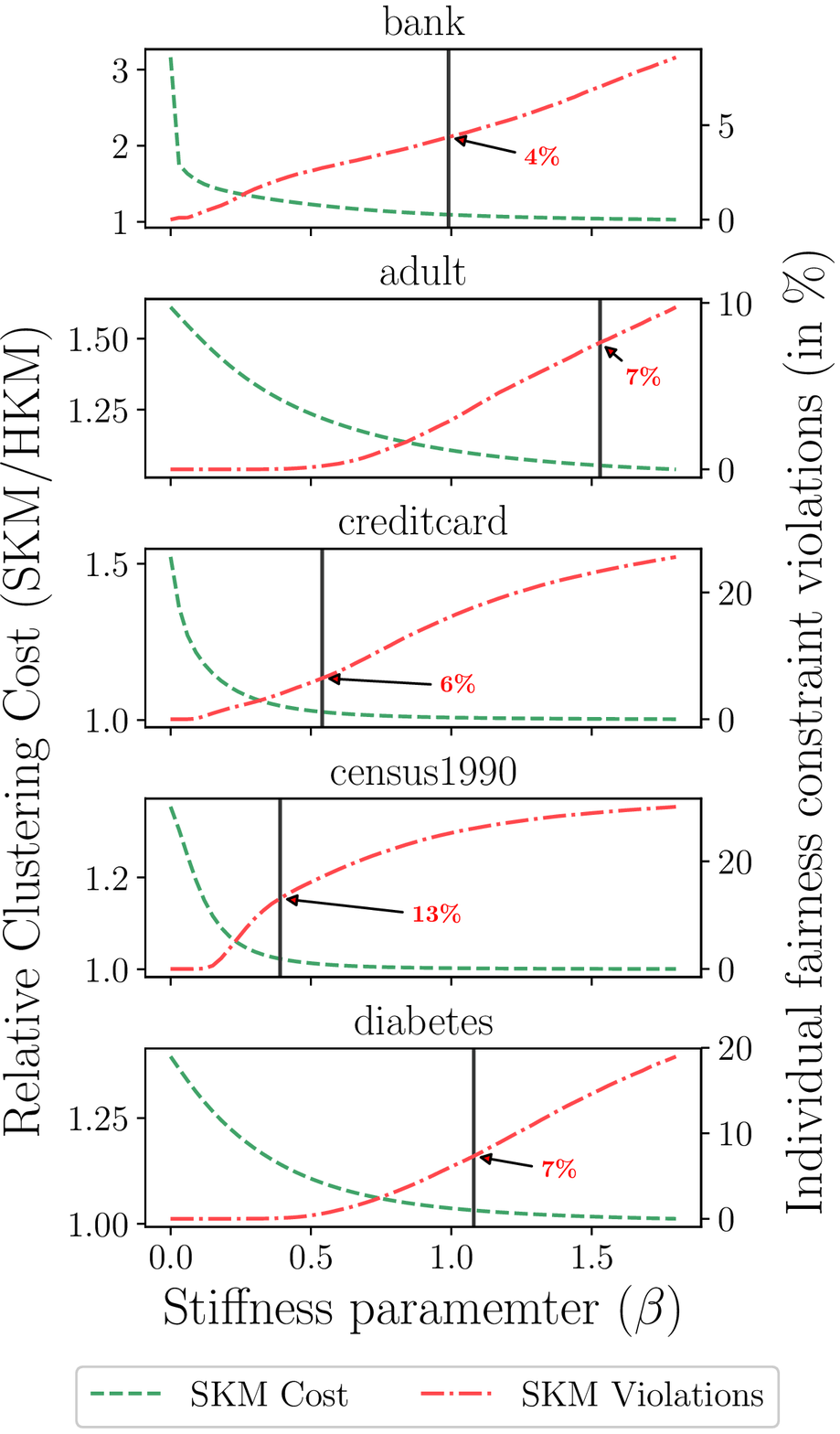}
        \label{app:fig:skmviolations_local}
    }
    \caption{Percentage of individual fairness constraint violations and relative clustering cost of $\skm$ vs stiffness parameter $\beta$ for $k = 8$. The vertical black line shows the $\beta$ at which $\skm$ and $\algif$ incur the same clustering cost.}
    \label{app:fig:skmviolations}
\end{figure}

{\bf Unfairness of $\boldsymbol{\skm}$. }
The output of $\skm$ depends on stiffness parameter $\beta$ introduced in \cite{skm-lecture-notes}. More specifically, when $\beta = 0$, we get a uniform distribution over the centers, which guarantees individual fairness at a very high cost. On the other hand, when $\beta \to \infty$, we get a low cost $\hkm$ solution, which is unfair to individuals. In~\Cref{app:fig:skmviolations}, we show the variation of clustering cost and percentage of individual fairness constraints violated ($\fairness_1$ and $\fairness_2$) by $\skm$ for different values of $\beta$. In~\Cref{app:tab:skmviolations} (extension of~\Cref{tab:skmviolations}), we find $\beta$ at which $\skm$ and $\algif$ incur the same clustering cost and observe the percentage of individual fairness constraints violated ($\fairness_1$ and $\fairness_2$). Note that $\fairness_2$ is a much relaxed fairness measure compared to $\fairness_1$: for each point, similarity is measured locally, with respect to its $\floor{{|V|}/{k}}$ nearest neighbors. Even with such relaxations, $\skm$ exhibits unfair treatment of similar points. Our solution does not violate any individual fairness constraints.

% ALG-IF OPT-IF
\begin{figure}[ht]
    \centering
    \subfloat[][Fairness similarity $\calF_1$] {
        \includegraphics[width=0.45\linewidth]{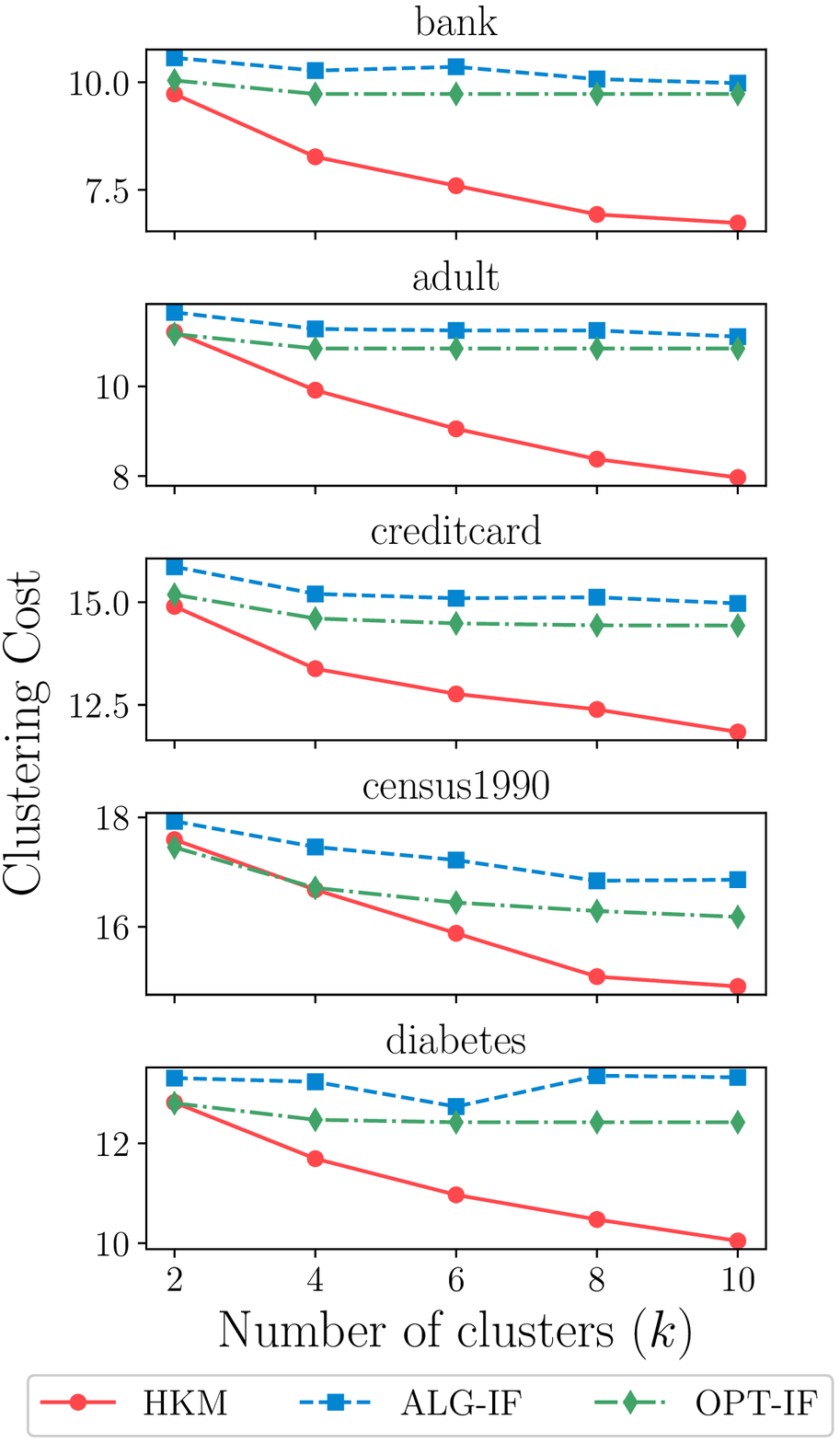}
        \label{app:fig:if_alg_opt_global}
    }
    \qquad
    \subfloat[][Fairness similarity $\calF_2$] {
        \includegraphics[width=0.45\linewidth]{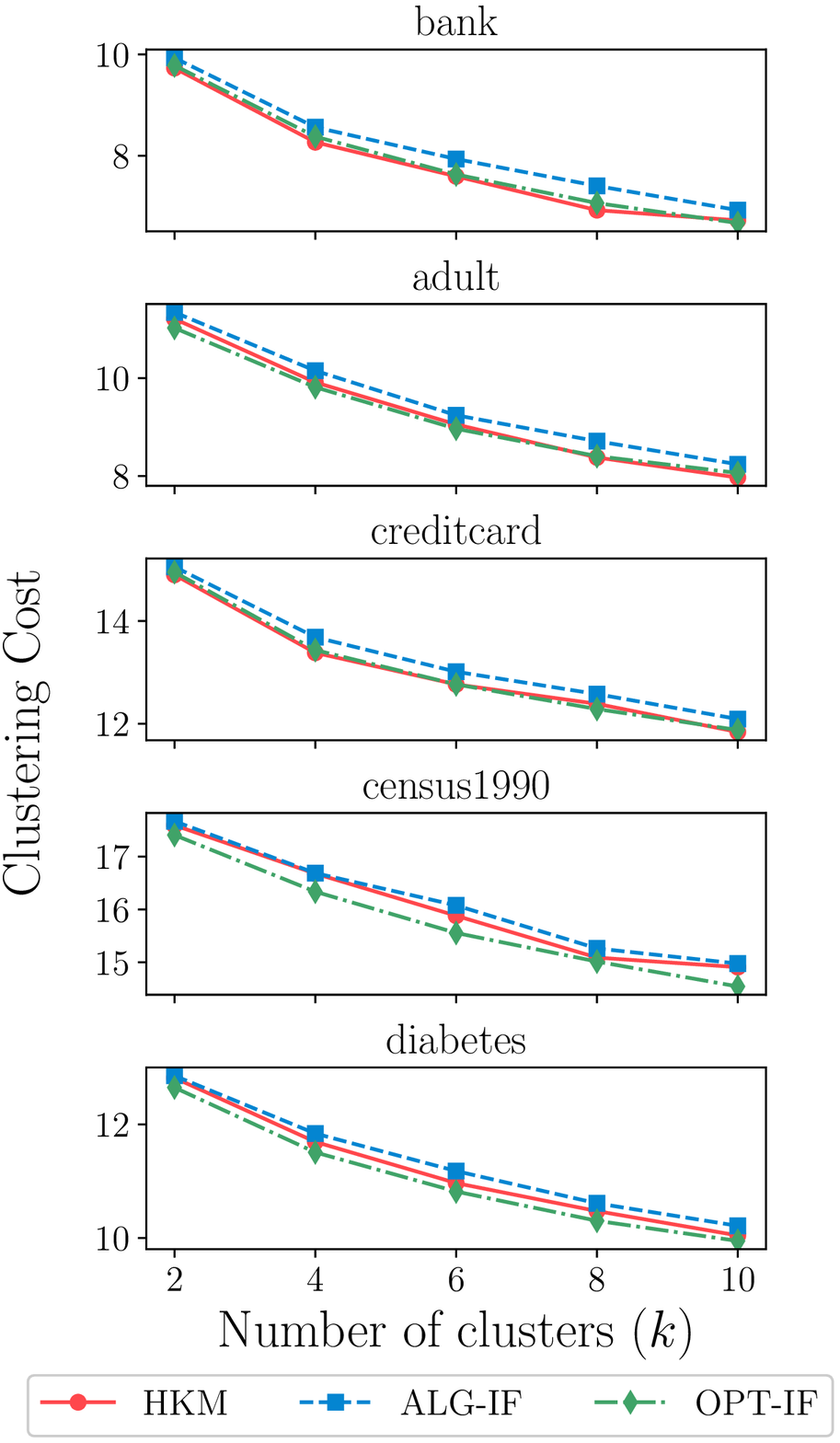}
        \label{app:fig:if_alg_opt_local}
    }
    \caption{Clustering cost vs number of clusters for $\algif$, $\optif$ and $\hkm$.}
    \label{app:fig:if_alg_opt}
\end{figure}

{\bf Cost Analysis. }
In this section, we present the plots for all the datasets, comparing the cost of $\algif$ and $\algtf$ against $\optif$ and $\opttf$, respectively, as shown in~\Cref{app:fig:if_alg_opt} and~\Cref{app:fig:cf_alg_opt} (extension of~\Cref{fig:if_alg_opt} and~\Cref{fig:cf_alg_opt}).

% ALG-CF OPT-CF
\begin{figure}[ht]
    \centering
    \subfloat[][Fairness similarity $\calF_1$] {
        \includegraphics[width=0.45\linewidth]{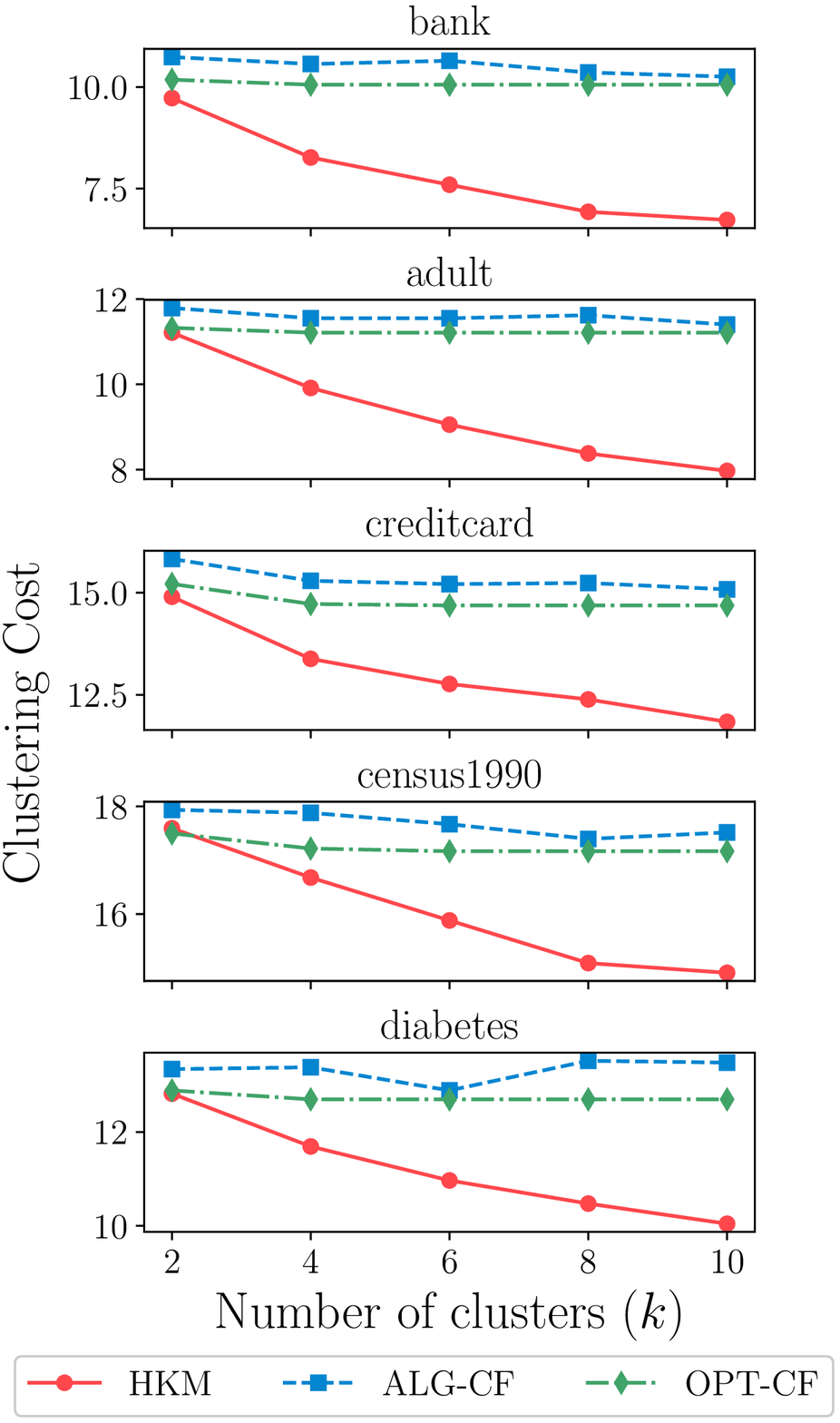}
        \label{app:fig:cf_alg_opt_global}
    }
    \qquad
    \subfloat[][Fairness similarity $\calF_2$] {
        \includegraphics[width=0.45\linewidth]{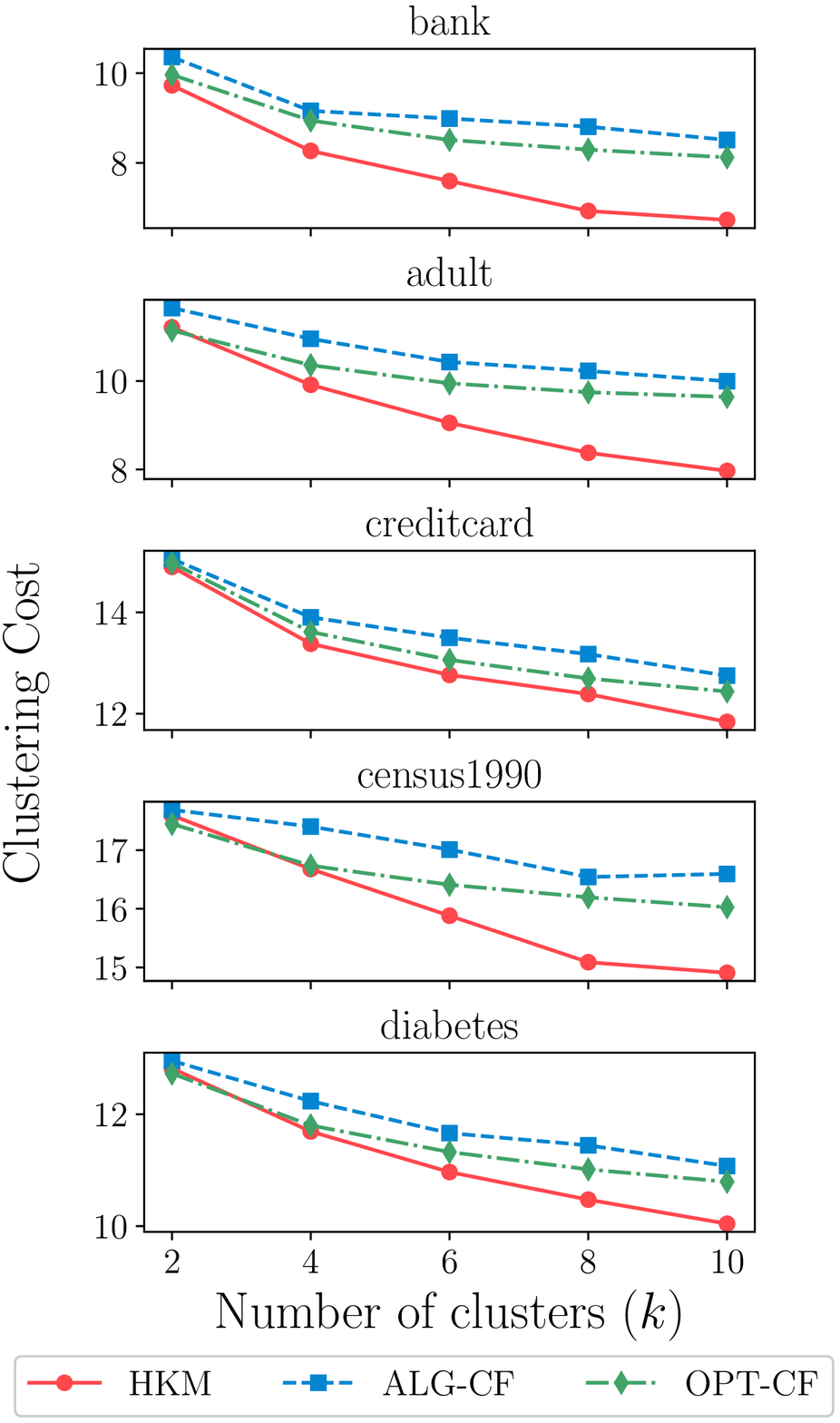}
        \label{app:fig:cf_alg_opt_local}
    }
    \caption{Clustering cost vs number of clusters for $\algtf$, $\opttf$ and $\hkm$.}
    \label{app:fig:cf_alg_opt}
\end{figure}

\begin{figure}[ht]
    \centering
    \includegraphics[width=\linewidth]{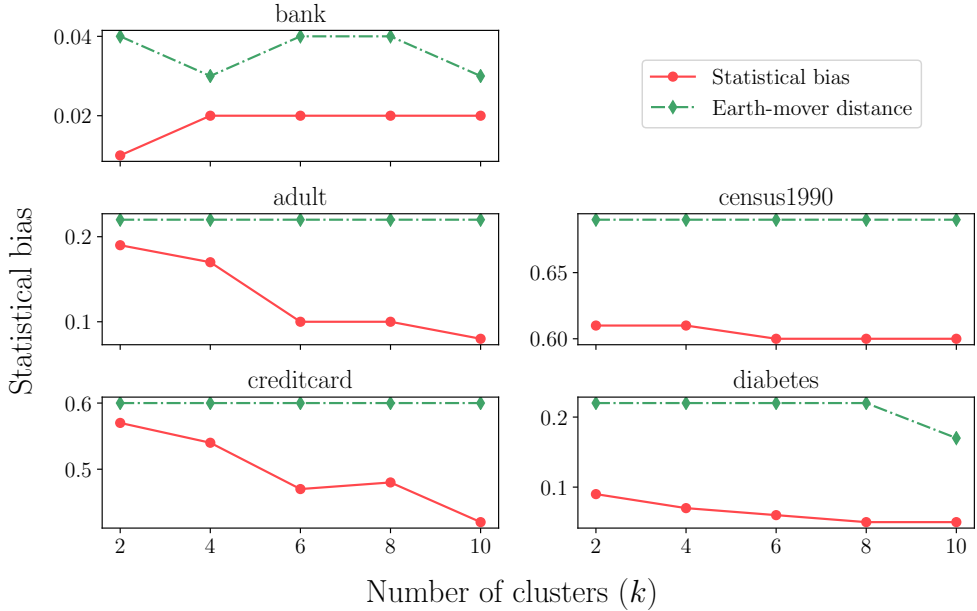}
    \caption{Variation of statistical bias and earth-mover distance vs number of clusters ($k$). We observe that statistical bias is upper bounded by earth-mover distance as suggested by~\Cref{lem:relation} and the gap is tight in practice.}
    \label{app:fig:biasearthmover}
\end{figure}

{\bf Individual Fairness to Group Fairness under $\fairness_1$. }
In this experiment, we use $\fairness_1$ as fairness similarity measure and $D_{TV}$ as statistical distance measure. Let $r = \argmax_r \frac{\MAD_r}{|G_r|}$. We plot statistical bias defined by $\frac{\MAD_r}{|G_r|}$ and the corresponding earth-mover distance $d_{\EM}(\nu_{G_r},\nu_V)$ as shown in~\Cref{app:fig:biasearthmover}. As~\Cref{lem:relation} suggests, we observe that $\frac{\MAD_r}{|G_r|} \leq d_{\EM}(\nu_{G_r},\nu_V)$. Moreover, the gap between statistical bias and earth-mover distance is tight in practice.

\end{document}